\documentclass{article}

\usepackage{microtype}
\usepackage{graphicx}
\usepackage{subfigure}
\usepackage{booktabs} 
\usepackage{enumitem}

\usepackage{hyperref}

\usepackage[accepted]{icml2024}

\usepackage{amsmath}
\usepackage{amssymb}
\usepackage{mathtools}
\usepackage{amsthm}

\usepackage[capitalize,noabbrev]{cleveref}

\theoremstyle{plain}
\newtheorem{theorem}{Theorem}[section]

\newtheorem{lemma}[theorem]{Lemma}

\theoremstyle{definition}

\theoremstyle{remark}

\usepackage[textsize=tiny]{todonotes}
\newcommand{\francesco}[1]{}{}
\newcommand{\vincent}[1]{}{}

\usepackage{multirow}
\usepackage{multicol}
\usepackage{thmtools}
\usepackage{thm-restate}
\usepackage{wrapfig}   
\usepackage{caption}
\usepackage{placeins}

\icmltitlerunning{Learning Useful  Representations of Recurrent Neural Network Weight Matrices}

\begin{document}

\twocolumn[
\icmltitle{Learning Useful  Representations of \\ Recurrent Neural Network Weight Matrices}



\icmlsetsymbol{equal}{*}

\begin{icmlauthorlist}
\icmlauthor{Vincent Herrmann}{idsia}
\icmlauthor{Francesco Faccio}{idsia,kaust}
\icmlauthor{Jürgen Schmidhuber}{idsia,kaust}
\end{icmlauthorlist}

\icmlaffiliation{idsia}{The Swiss AI Lab IDSIA, USI \& SUPSI}
\icmlaffiliation{kaust}{AI Initiative, KAUST}

\icmlcorrespondingauthor{Vincent Herrmann}{vincent.herrmann@idsia.ch}

\icmlkeywords{Machine Learning, ICML}

\vskip 0.3in
]



\printAffiliationsAndNotice{}  

\begin{abstract}
Recurrent Neural Networks (RNNs) are general-purpose parallel-sequential computers. 
The program of an RNN is its weight matrix. 
How to learn useful representations of RNN weights that facilitate RNN analysis as well as downstream tasks? 
While the \textit{mechanistic approach} directly looks at some RNN's weights to predict its behavior, the \textit{functionalist approach} analyzes its overall functionality---specifically, its input-output mapping. 
We consider several mechanistic approaches for RNN weights and adapt the permutation equivariant Deep Weight Space layer for RNNs. 
Our two novel functionalist approaches extract information from RNN weights by `interrogating' the RNN through probing inputs. 
We develop a theoretical framework that demonstrates conditions under which the functionalist approach can generate rich representations that help determine RNN behavior. 
We release the first two `model zoo' datasets for RNN weight representation learning. 
One consists of generative models of a class of formal languages, and the other one of classifiers of sequentially processed MNIST digits.
With the help of an emulation-based self-supervised learning technique we compare and evaluate the different RNN weight encoding techniques on multiple downstream applications.
On the most challenging one, namely predicting which exact task the RNN was trained on, functionalist approaches show clear superiority.
\end{abstract}

\section{Introduction}
\label{introduction}
For decades, researchers have developed techniques for learning representations of complex objects such as images, text, audio and video with deep neural networks (NNs). 
This expertise has significantly advanced the field by enabling models to convert data into formats useful for solving problems. 
In particular, recurrent NNs (RNNs) have been widely adopted due to their computational universality~\cite{siegelmann91turing}.
Low-dimensional representations of the programs of RNNs (their weight matrices) are of great interest as they can speed up the search for solutions to given problems. 
For instance, compressed representations of RNN weight matrices have been used to evolve  RNN parameters~\cite{koutnik:gecco10} for controlling a car from raw video input~\cite{koutnik2013evolving},
using Fourier-type transforms, e.g., the coefficient of the Discrete Cosine Transform (DCT)~\cite{Srivastava:2012:GCN:2330784.2330902}, without using the capabilities of NNs to learn such representations.
Recent work has seen a rise of representation learning techniques for NN weights using NNs as encoders~\cite{eilertsen2020classifying, Unterthiner2020PredictingNN, schurholt2021self, dupont2022data, faccio2022general}. However, there is a lack of methods for learning representations of RNNs. 

This paper introduces novel techniques for learning RNN representations using powerful NNs, which may be RNNs themselves. Just like representation learning in other fields, such as computer vision, facilitates solutions of specific tasks, such techniques can facilitate learning, searching, and planning with RNNs.
We show that by employing general RNN weight encoder architectures and self-supervised learning methods, it is possible to learn representations that capture diverse functionalities of RNNs. We differentiate between encoders that treat the weights as input data (mechanistic) and those that engage only with the function defined by the weights (functionalist). Within the functionalist approach, the \textit{non-interactive probing} method uses learnable but fixed probing sequences as input to the RNN and observes the corresponding outputs. In contrast, \textit{interactive probing} adapts the probing sequences dynamically based on the input RNN in order to extract the most relevant information. We provide empirical and theoretical evidence of the effectiveness of interactive probing for complex tasks, despite occasional training stability issues. For simpler tasks or when interactive properties are not required, non-interactive probing or mechanistic encoders might be more suitable.

\textbf{Our contributions are summarized as follows:}
\vspace{-10pt}
\begin{enumerate}[label=(\arabic*), itemsep=0pt]
    \item We introduce the challenge of learning useful representations of RNN weights and propose six neural network architectures for processing these weights. We define the difference between mechanistic and functionalist approaches, adapt Deep Weight Space Nets (DWSNets, \citet{navon2023equivariant, zhou2023permutation}) to RNNs, and introduce novel probing architectures, including the concept of interactive probing.
    \vspace{-4pt}
    \item We develop a theoretical framework for analyzing the efficiency of interactive and non-interactive probing encoders.
We prove that interactive probing encoders can be exponentially more efficient for certain problems.
\vspace{-4pt}
    \item We create and release two comprehensive RNN ``model zoo'' datasets. Each dataset consists of the weights of thousands of LSTMs~\cite{Hochreiter:97lstm}, trained on hundreds of different but related tasks. 
One dataset focuses on formal languages, while the other on tiled sequential MNIST.
\footnote{\url{https://github.com/vincentherrmann/rnn-weights-representation-learning}}
\vspace{-4pt}
    \item We conduct empirical analyses and comparisons across the different encoder architectures using these datasets, showing which encoders are more effective.
\end{enumerate}

\section{Related Work}

The concept of learning representations for the weights of feedforward NNs, also sometimes called hyper-representations, has been explored in studies by \citet{eilertsen2020classifying} and \citet{schurholt2021self}.
Various methods for processing NN weights have been proposed.
\citet{schurholt2021self, eilertsen2020classifying, faccio2020parameter, herrmann2022learning} suggest flattening the weights and using them as input data for simple encoders or predictors.
\citet{Unterthiner2020PredictingNN} and \citet{tang2022inputting} use permutation-invariant layers to extract high-level weight features.
\citet{navon2023equivariant} and \citet{zhou2023permutation} develop weight processing layers that are equivariant to neuron permutation, leading to the creation of DWSNet architectures, which can universally approximate functions of the weight space.
Other approaches include probing NNs with learnable inputs and analyzing the network based on the generated outputs, as proposed by
\citet{learningtothink2015, harb2020policy} and \citet{faccio2022goal, faccio2022general}.
For processing implicit neural representations, \citet{dupont2022data} employ normalizing flows and diffusion models, while \citet{xu2022signal} use higher-order spatial derivatives. All mentioned works, except for \citet{learningtothink2015} and \citet{herrmann2022learning}, focus solely on the processing of feedforward (including convolutional) NNs.

Emulation as an objective for learning representations was proposed by \citet{raileanu2020fast}, but their focus was on policy trajectories rather than on NN weights. For self-supervised representation learning of NN weights, reconstruction-based approaches have been explored \cite{schurholt2021self, dupont2022data}.
\citet{ramesh2021model} employ populations of models trained on diverse tasks for continual learning. Both \citet{eilertsen2020classifying} and \citet{schurholt2022model} have released datasets of trained convolutional NNs. To our knowledge, there are no public datasets of diverse trained RNNs.

\section{RNN Weight Encoders}
\label{sec:rnn_weight_encoders}

\begin{figure*}[h]
\centering
\includegraphics[width=1.0\textwidth]{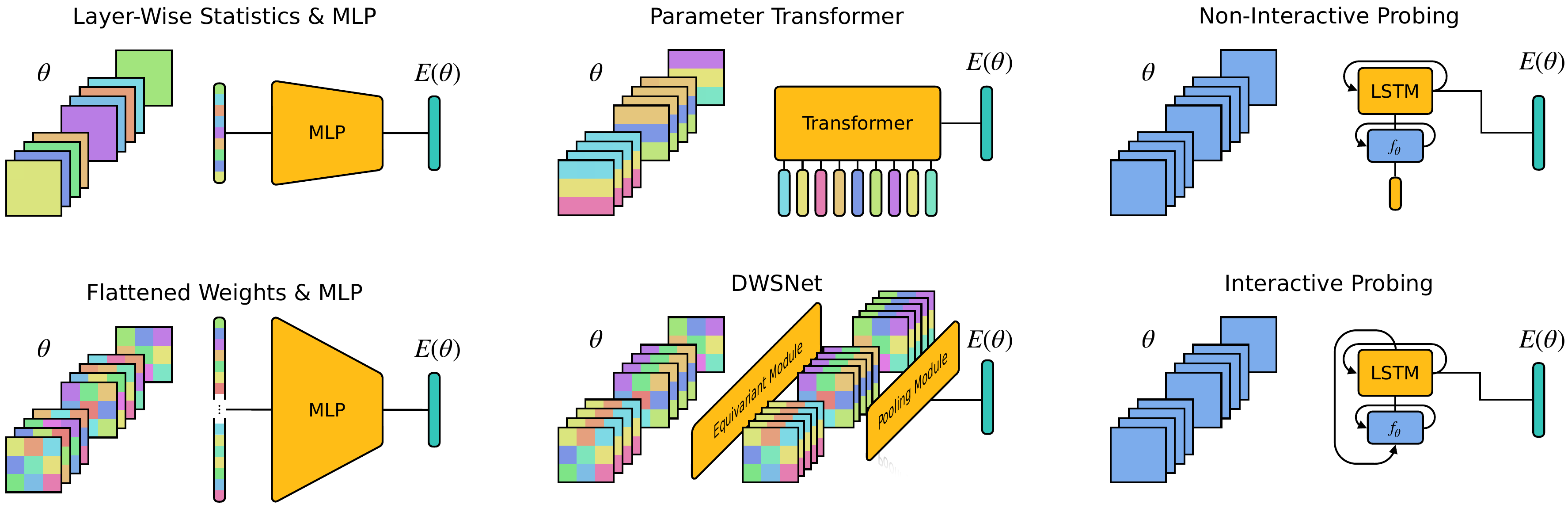}
\caption{RNN weight encoder architectures taking weights $\theta$ as input and producing a representation $E(\theta)$. The two groups of four weight matrices symbolize the four gates of two LSTM layers. The last matrix represents the output projection.}
\label{fig:encoder_architectures}
\end{figure*}

When using NN weights as input for another network, two main challenges arise.
First, the number of weights can quickly become very large.
Second, the weight space exhibits symmetries, particularly with respect to the permutation of hidden neurons \cite{hecht1990algebraic}. 
Rearranging the order of hidden neurons in a network does not change the computation it performs\footnote{With piece-wise linear activation functions like ReLU, there is also invariance to certain types of weight scaling. However, this is not a focus here as RNNs usually use different activation functions. There is another symmetry with respect to sign flips.}. An effective NN weight encoder should recognize these symmetries.

We consider an RNN, $f_\theta: \mathbb{R}^X \times \mathbb{R}^H \to \mathbb{R}^Y  \times \mathbb{R}^H; (x, h_{t-1}) \mapsto (y, h_t)$, $t \in 1, 2, \dots$, parametrized by $\theta \in \Theta$, which maps an input $x$ and hidden state $h_{t-1}$ to an output $y$ and a new hidden state $h_t$.
In the following, we assume that all RNNs are multi-layer LSTMs\footnote{Generalizing our framework to other RNN architectures should be straightforward.}.
An RNN weight encoder is a function $E_{\phi}: \Theta \to \mathbb{R}^Z; \theta \mapsto z$, mapping RNN weights $\theta$ to $Z$-dimensional representation vectors $z = E_\phi(\theta)$.
The encoder's parameters are $\phi \in \Phi$.

We differentiate between two approaches for encoding RNN weights:
(1) \textit{Mechanistic} encoders ``look'' at the weights $\theta$ directly, treating them as typical input data. 
(2) \textit{Functionalist} encoders, instead, interact with the function $f_\theta$ without direct access to the weights themselves. 
These encoders still map RNN weights to representations, but focus on a functional interpretation of the RNN.
We discuss and compare six encoder architectures for representing RNN weights, as depicted in Figure~\ref{fig:encoder_architectures}. 

\paragraph{Layer-Wise Statistics}
This approach, successfully used in previous studies \cite{Unterthiner2020PredictingNN} to predict properties of CNNs, involves creating, for each weight matrix, a vector consisting of mean, standard deviations, and five quantiles ($0, 0.25, 0.5, 0.75, 1$).
For LSTMs, each layer yields twelve distinct vectors, for each of the four gates corresponding to the input-to-hidden weights, hidden-to-hidden weights, and the bias vector.
These vectors are then concatenated and given to a multi-layer perceptron (MLP). The architecture is inherently invariant to permutations of hidden neurons. It efficiently scales with RNN size due to its reliance on high-level features.
However, this invariance extends to many transformations beyond those that preserve the RNN functionality. Consequently, RNNs with identical layer-wise statistics may behave differently, preventing the encoder's ability to approximate all functions of the weight space.

\paragraph{Flattened Weights}
In this approach, all RNN weights are flattened into a single vector before being fed into an MLP. Unlike the previous architecture, this one is not invariant to hidden neuron permutations. As a result, it faces a different challenge: numerous RNNs might appear different to the encoder yet perform identical computations.
This impedes generalization, as the MLP has difficulty learning these symmetries (empirically demonstrated in Appendix~\ref{app:permutation_invariance}).
Another issue is the very large size of the input vector.
Let $N$ represent the number of hidden neurons in $f_\theta$.
The number of parameters in the input layer of the MLP is proportional to $N^2$.
However, with an adequately large MLP, this architecture can approximate any function in the weight space\footnote{When mentioning universality of weight space functions $\Theta \to \mathbb{R}^Z$, we imply the regularity conditions of \citet{navon2023equivariant}, Proposition 6.2.}.
This follows immediately from the universal approximation property of MLPs \cite{hornik1991approximation}.

\paragraph{Parameter Transformer}
\citet{schurholt2021self} introduced an attention-based architecture. This design treats the weights of individual neurons (specifically, the rows of weight matrices along with their corresponding bias values) as a sequence. These sequences are then processed by an encoder-only transformer model \cite{vaswani2017attention}. 
A learned positional encoding ensures that the transformer has the information which weights correspond to which neuron. 
This also makes it not invariant to neuron permutation. 
The attention mechanism within the transformer enables associative retrieval of information from other neurons, which could be a beneficial inductive bias when handling NN weights.
The size of both the neuron sequence and the input transformation parameters scale linearly with $N$.
Given that transformers are known to be universal sequence-to-sequence function approximations \cite{yun2019transformers}, the parameter transformer can theoretically approximate any function of the weight space.

\paragraph{DWSNet}
\label{sec:DWSNet}
Both \citet{navon2023equivariant} and \citet{zhou2023permutation}  proposed architectures for processing the weights of feedforward NNs, closely related in design. These architectures are invariant precisely to permutations of hidden neurons and are capable of universally approximating functions of the weight space. We refer to this architecture as DWSNet, following \citet{navon2023equivariant}, and extend its application to LSTM networks.
The central concept of DWSNet is to construct layers that are equivariant to the hidden neuron permutation group. These layers process the weights, where each weight is represented by a feature vector.
A final pooling layer across all weights ensures invariance to neuron permutation. The mechanisms for adapting DWSNets to LSTMs, along with arguments for their universality, are presented in Appendix~\ref{app:dwsnet}.
Appendix~\ref{app:permutation_invariance} validates the implementation by demonstrating that DWSNets maintain invariance only to correctly permuted weights in LSTMs.

\paragraph{Non-Interactive Probing}

\begin{figure}[t]
\centering
\includegraphics[width=0.9\linewidth]{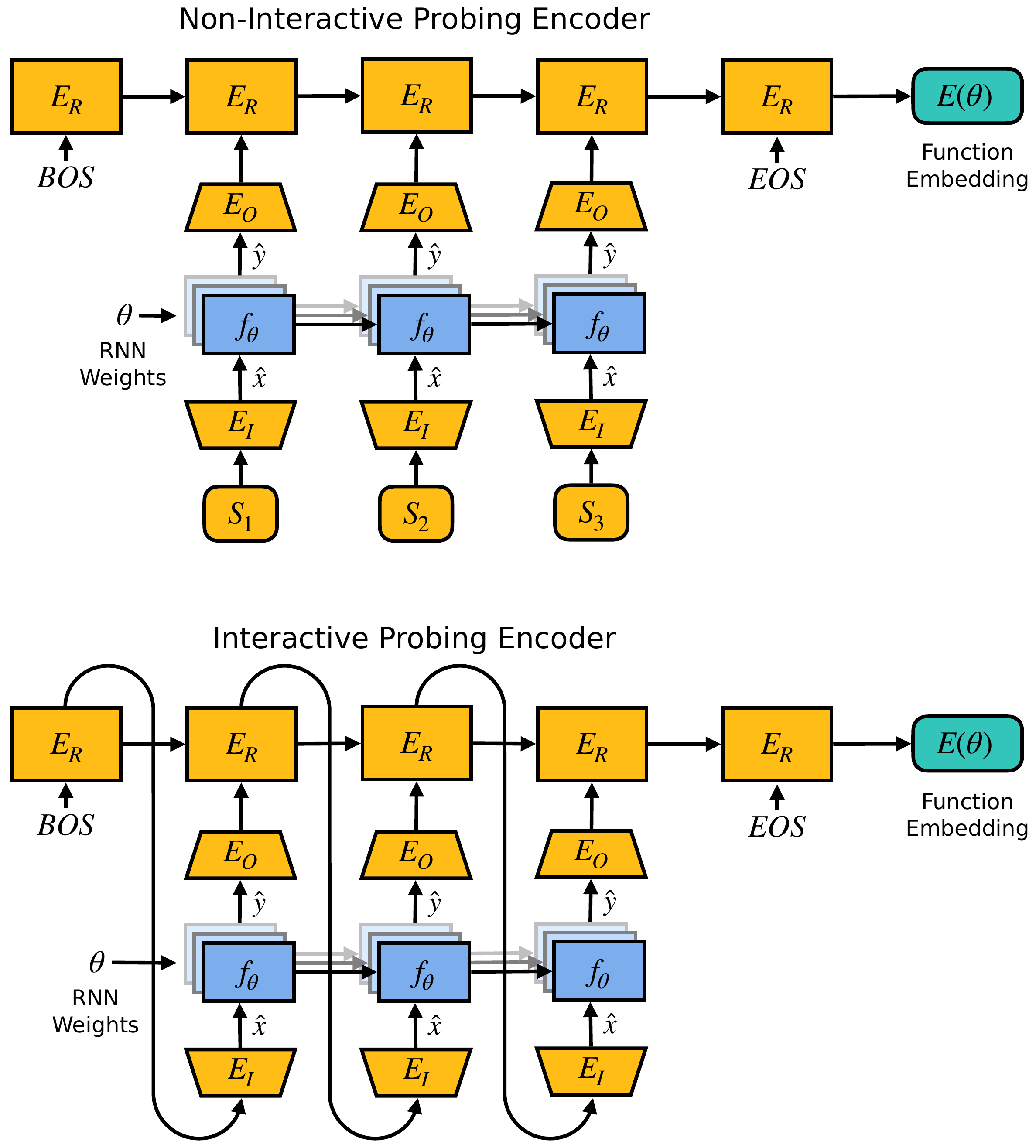}
\caption{Comparison of non-interactive and interactive procing encoders.}
\label{fig:probing_encoders}
\end{figure}

In the context of Reinforcement Learning and Markov Decision Processes, policy fingerprinting has emerged as an effective method to evaluate feedforward NN policies~\cite{harb2020policy, faccio2022goal, faccio2022general}. 
In policy fingerprinting, a set of learnable probing inputs is given to the network. 
Based on the set of corresponding policy outputs, a function (policy) representation is produced. 
This functionalist approach, which we refer to as non-interactive probing in this paper, is easily adaptable to RNNs by employing sequences of probing inputs (see Figure~\ref{fig:probing_encoders}, top). 

The probing sequence length, denoted as $L$, is a fixed hyperparameter.
A non-interactive probing encoder uses a sequence of learnable embeddings, $S_i, 1 \leq i \leq L$.
At each probing step $i$, the embedding $S_i$ is transformed by $E_I$, an MLP followed by a reshaping operation that outputs a set of $M$ probing inputs $\hat{x}_{im} \in \mathbb{R}^X \; \forall m \in \{1, \dots, M\}$.
$M$ represents the number of parallel probing sequences. 
These inputs are processed as a batch by the RNN $f_\theta$, resulting in a batch of probing outputs $\hat{y}_{im} = f_\theta(\hat{x}_{im}) \in \mathbb{R}^Y$
\footnote{For $f_\theta$, the recurrence of the hidden states over sequence steps is implied.}.
The probing outputs are concatenated and further transformed by another MLP $E_O$, resulting in a vector $o_i$.
The sequence of vectors $o_i$ is processed by an LSTM network $E_R$, which then outputs the RNN representation $E(\theta)$.

\paragraph{Interactive Probing}
The probing sequences for non-interactive probing are static, i.e., at test time, the probing sequences do not depend on the specific RNN being evaluated. 
The alternative is to make the probing sequences dynamically dependent on the given RNN.
Each new probing input, $\hat{x}_i$, should depend on the previous probing outputs $\hat{y}_{<i}$. 
This dynamic adaptation is achieved by feeding the output of the previous step's LSTM $E_R(o_{<i})$ into $E_I$ (Figure~\ref{fig:probing_encoders}, bottom).
A similar concept has been proposed by \citet{learningtothink2015} for extracting algorithmic information from recurrent world models.

Both types of probing encoders retain the invariance properties of $f_\theta$.
However, functionalist encoders have limitations in differentiating between weight space functions; they cannot discern mechanistic differences in functionally equivalent RNNs. Consequently, two RNNs performing exactly the same function will look identical to a probing encoder, even if they use different algorithms to compute the function.
This means probing encoders are not universal in the sense that some of the mechanistic approaches are.
Table~\ref{tab:encoders_theoretical_properties} summarizes key properties of the different encoder architectures.

\begin{figure*}[t]
    \centering
    \begin{minipage}[b]{0.67\linewidth}
        \centering
        \captionof{table}{Properties of the different RNN weight encoder architectures. $N$ is the number of hidden neurons in $f_\theta$.}
        \label{tab:encoders_theoretical_properties}
        \begin{center}
        \begin{small}
        \begin{tabular}{lcccc}
        \toprule
        \textbf{Encoder} & \multicolumn{1}{c}{\begin{tabular}[c]{@{}c@{}}\textbf{Permutation}\\ \textbf{Invariant}\end{tabular}} & \multicolumn{1}{c}{\begin{tabular}[c]{@{}c@{}}\textbf{Universal}\\ \textbf{Approx.}\end{tabular}} & \textbf{\#Params} & \textbf{Type} \\
        \midrule
        Layerwise Statistics & Yes & No & const. & Mechanistic \\ 
        Flattened Weights & No & Yes & $ O(N^2) $ & Mechanistic \\ 
        Parameter Transformer & No & Yes & $ O(N) $ & Mechanistic \\ 
        DWSNet & Yes & Yes & const. & Mechanistic \\ 
        Non-Interactive Probing & Yes & No & const. & Functionalist \\ 
        Interactive Probing & Yes & No & const. & Functionalist \\
        \bottomrule
        \end{tabular}
        \end{small}
        \end{center}
        \vskip -0.1cm
    \end{minipage}\hfill
    \begin{minipage}[b]{0.307\linewidth}
        \centering
        \includegraphics[width=0.7\linewidth]{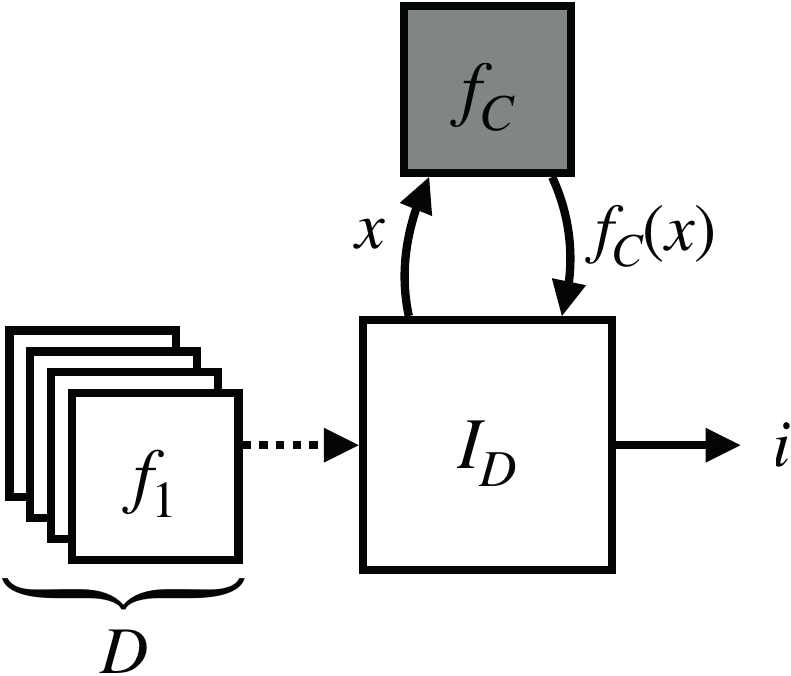}
        \vspace{-0.2cm}
        \captionof{figure}{Interrogator $I_D$ has access to a set $D$ of functions and interacts with function $f_C$, which it has to identify.}
        \label{fig:interrogator}
    \end{minipage}
\end{figure*}

\subsection{Theoretical Aspects of the Functionalist Approach}

We introduce a theoretical framework for analyzing probing encoders and the distinctions between interactive and non-interactive settings.
In practice, any RNN weight encoder is trained on a finite number of distinct RNNs.
For each of them, it should output a unique representation.
In this section, we study the related task of uniquely identifying a function from a given set by interacting with it.
Rather than using RNNs, we examine total computable functions, which are functions that halt and produce an output for every input. 
This is a minor limitation, since in practical scenarios, RNNs are almost always given a finite runtime.

Let $D$ represent a set of $n$ total computable functions $\{f_i : \mathbb{N} \to \mathbb{N} | i = 1, 2, \ldots, n\}$. 
In other words, $D$ comprises $n$ Turing machines that halt on every input, with no pair being functionally equivalent.
Let $I_D$ denote another Turing machine, which we call the \textit{Interrogator}. 
$I_D$ has access to the function set $D$ (e.g., the corresponding Turing numbers might be written somewhere on its tape). 
Moreover, $I_D$ is given access to one function $f_C \in D$ as a black box. 
$I_D$ can interact with $f_C$ by providing an input $x \in \mathbb{N}$ and subsequently reading the corresponding output $f_C(x)$. 
The task of $I_D$ is to identify which member of $D$ corresponds to function $f_C$, while minimizing interactions with $f_C$.
Specifically, $I_D$ must return $i \in \{1, \dots, n\}$ such that $f_C = f_i$.
This setup is depicted in Figure~\ref{fig:interrogator}.
It should be mentioned that RNNs of finite size and precision are \textit{not} universal computers~\cite{merrill2019sequential, deletang2022neural}.
However, the following propositions depend mainly on the relative computational ability of $I_D$ and the functions in $D$.
Hence we choose this simple abstract framework of distinguishing total computable functions and believe that it transfers well into realistic RNN settings.
The proofs of the following propositions can be found in Appendix~\ref{app:proofs}.

\begin{restatable}{proposition}{ab}
\label{th:interaction_limit}
Any function $f_C$ from a set $D$ can be identified by an interrogator through at most $|D| - 1$ interactions.
\end{restatable}

An Interrogator is called \textit{interactive} if the value $x_j$ of the $j$th probing input depends on $f_C(x_1), \dots, f_C(x_{j-1})$, i.e., the outputs corresponding to the previous probing inputs.
This implies that the probing inputs generally depend on the specific function $f_C$ given to $I$.
Conversely, a \textit{non-interactive} Interrogator can only provide a fixed set of probing inputs to $f_C$, and their values do not depend on the outputs of $f_C$.
In the proof of Proposition~\ref{th:interaction_limit}, the probing inputs given to $f_C$ do not dynamically depend on $f_C$.
This means that the theorem holds for non-interactive Interrogators.
A natural question arises: Can interactive Interrogators identify a function using fewer interactions? Although there are instances where they need exponentially fewer interactions, in the worst-case scenario, both methods necessitate an equivalent number of interactions:

\begin{restatable}{proposition}{cd}
The upper bound for probing interactions required to identify a function from a given function set $D$ is $|D|-1$ for both interactive and non-interactive Interrogators.
\end{restatable}

\begin{restatable}{proposition}{ef}
\label{prop:fewer_interactions}
There exist function sets for which an interactive Interrogator requires exponentially fewer probing interactions to identify a member than does a non-interactive one.
\end{restatable}

Section~\ref{sec:experiments} demonstrates that these theoretical concepts are mirrored in empirical results. In one dataset (formal languages), interactive probing significantly outperforms non-interactive probing. However, in another dataset, both methods show similar performance.

\section{Self-Supervised Learning of RNN Weight Representations}
\label{sec:self_supervised_learning}

We propose a general-purpose method for learning representations of RNN weights.
It is based on the idea that the RNN weight representation should contain all the information necessary in order to emulate the RNN's functionality.
A very similar technique is used by \citet{raileanu2020fast} to learn representations of (non-recurrent) policies based on their trajectories.

The RNN $f_\theta$ interacts with a potentially stochastic environment, 
$\mathcal{E}$, that maps an RNN's output $y$ to a new input $x$. 
The environment may have its own hidden state $\eta$. 
By sequentially interacting with the environment, the RNN produces a rollout defined by:
\begin{equation*}
\left\{
\begin{aligned}
x_t, \eta_t &= \mathcal{E}(y_{t-1}, \eta_{t-1}) \\
y_t, h_t &= f_\theta(x_t, h_{t-1}),
\end{aligned}
\right.
\end{equation*}
with fixed initial states $y_0, \eta_0$ and $h_0$. 
For instance, $f_\theta$ might be an autoregressive generative model, with $\mathcal{E}$ acting as a stochastic environment that receives a probability distribution over some language tokens, $y_t$---the output of $f$ at time step $t$---, and produces a representation (e.g., a one-hot vector) of the new input token $x_{t+1}$.
When the environment is stochastic, numerous rollouts can be generated for any $\theta \in \Theta$. 
A rollout sequence of a function $f_\theta$ in environment $\mathcal{E}$ has the form $S_\theta = (x_1, y_1, x_2, y_2, \dots)$.

To train an RNN weight encoder $E_{\phi}$, we consider an Emulator $A_{\xi}: \mathbb{R}^X \times \mathbb{R}^B \times \mathbb{R}^Z \to \mathbb{R}^Y \times \mathbb{R}^B; (x, b_{t-1}, z) \mapsto (\tilde{y}, b_t)$, parametrized by $\xi \in \Xi$. 
The Emulator is an RNN with hidden state $b$ that learns to imitate different RNNs $f_\theta$ based on their function encoding $z = E(\theta)$.

We consider a dataset $\mathcal{D} = \{(\theta_i, S_{\theta_i}) | i = 1, 2, \dots\}$ composed of tuples, each containing the parameters of a different RNN and a corresponding rollout sequence. 
We assume that all RNNs have the same initial state $h_0$ but have been trained on different tasks.
The Encoder $E_{\phi}$ and the Emulator $A_{\xi}$ are jointly trained by minimizing a loss function $\mathcal{L}$. 
This loss function measures the behavioral similarity between an RNN $f_{\theta}$ and the Emulator $A_{\xi}$, which is conditioned on the function representation $z=E_{\phi}(\theta)$ of $\theta$ as produced by the Encoder $E_{\phi}$ (see Figure~\ref{fig:pretraining}). 
Put simply, the Emulator uses the representations of a set of diverse RNNs $f_{\theta}$ to imitate their behavior:\footnote{The recurrence of $A_\xi$ is omitted for simplicity.}
\begin{equation}
\label{eq:loss}
    \min_{\phi, \xi} \mathbb{E}_{(\theta, S) \sim \mathcal{D}} \sum_{(x_i, y_i) \in S} \mathcal{L} \big( A_{\xi} (x_i, E_{\phi}(\theta)) , y_i \big).
\end{equation}
In the case of continuous outputs $y$, the mean-squared error provides a suitable loss function $\mathcal{L}$. 
For categorical outputs, we employ the reverse Kullback-Leibler divergence because of its mode-seeking behavior.

If we are interested only in the RNN weight encoder $E_\phi$, the emulator $A_\xi$ can be discarded after training. 
However, there might be potential applications for the emulator, for example, in the context of imitation learning~\cite{zare2023survey}, behavior cloning~\cite{torabi2018behavioral}, or consolidating the knowledge from many different models~\cite{kj2020meta}.

\begin{figure}[t]
\centering
\includegraphics[width=0.7\linewidth]{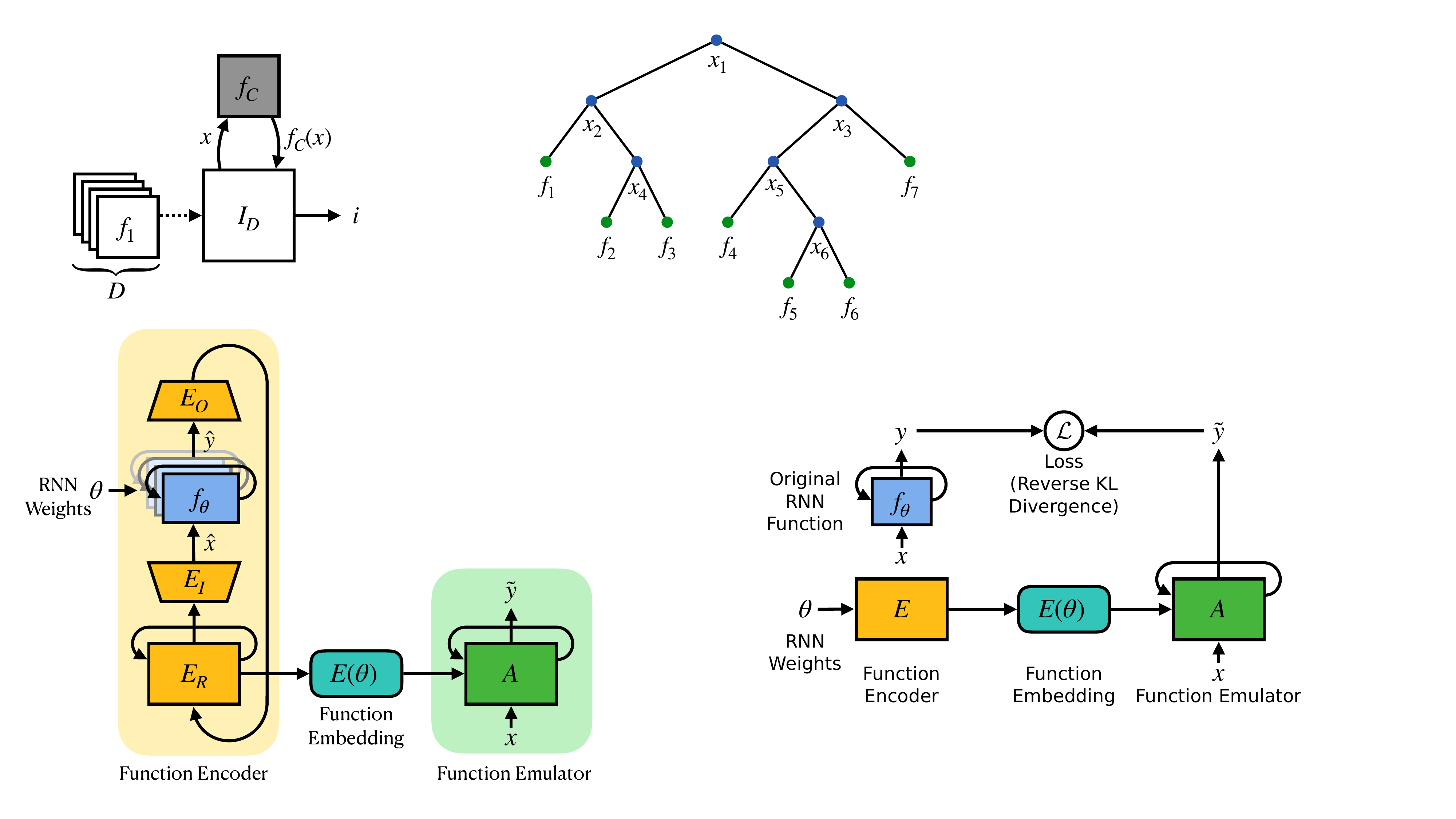}
\caption{Emulation-based self-supervised training. The encoder $E$ is trained to generate embeddings of $\theta$ that allow $A$ to emulate $f_\theta$.}
\label{fig:pretraining}
\end{figure}

\section{Datasets}
\label{sec:datasets}
To evaluate the methods described and foster further research, we develop and release two ``model zoo'' datasets for RNNs. The first dataset focuses on modeling formal languages, while the second is centered around predicting digits in a tiled Sequential MNIST format.
Both datasets share a similar structure.
We train 1000 LSTMs (each with two layers and a hidden state size of 32) on various tasks.
The weights $\theta$ of each model are saved at $9$ fixed training steps, along with 100 rollouts $S_\theta$ and additional data, such as the current performance on its task.
The datasets are divided into training, validation, and out-of-distribution (OOD) test splits, with tasks in each split being non-overlapping. The tasks in the OOD set are structurally slightly different.

\subsection{Formal Languages}
\label{sec:formal_languages}

The models are auto-regressive language models, trained on different formal languages using teacher forcing and the standard cross-entropy language modelling objective.
Let $a, b, c$ and $d$ be the four tokens of a language.
We define a language as $L_{m_b, m_c, m_d} := \bigl\{ a^{n} b^{n + m_b} c^{n + m_c} d^{n + m_d} | n \geq - \min \{0, m_b, m_c, m_d\} \bigl\}$.
This means that, in a string from such a language, the number of appearances of the tokens, relative to each other, is determined by
$m_b, m_c$ and $m_d$.
For example, the strings from the language $L_{1, -1, 2}$ are $\{\{abbddd\}, \{aabbbcdddd\}, \{aaabbbbccddddd\}, \dots\}$. 
Each model is trained on one language from the set $G_L := \{L_{m_b, m_c, m_d} | m_b, m_c, m_d \in \{-3, -2, \dots, 2\}\}$. 
Note that these languages are essentially ones used also in the proof of proposition~\ref{prop:fewer_interactions}.
In $G_L$, there are a total of $6^3 = 216$ different languages.

All models are trained on sequences of length 42, including one begin-of-sequence (BOS) and one end-of-sequence (EOS) token.
The maximum value of $n$ is 10.
If a language string is shorter than 42, it is padded at the end with EOS tokens.
The OOD test set contains the RNNs trained on languages where the sum of the absolute values of $m_b, m_c$ and $m_d$ is the smallest.

\subsection{Tiled Sequential MNIST}
\label{sec:sequential_mnist}

The models of this dataset are trained to classify MNIST digits presented in a sequential format.
Unlike the typical pixel-wise sequence, each digit is represented as a sequence of 49 $4\times4$ tiles (plus BOS and EOS tokens).
This approach improves computational efficiency for both RNN training and weight representation experiments.
After each tile of the sequence, the model predicts the digit.
The loss is the mean cross-entropy of all predictions in the sequence. However, the accuracy of each model is assessed based on the final prediction, i.e., when the model has seen the entire digit.
The dataset's task involves rotating MNIST digits. Each model is exposed to the entire MNIST dataset, with images rotated by a unique random angle. For the training and validation sets, the rotations range from 0 to 311 degrees, while for the OOD test set, they range from 312 to 360 degrees.

\section{Experiments and Results}
\label{sec:experiments}

Our empirical investigation involves two experimental phases. In the first phase, we apply the emulation-based representation method described in Section~\ref{sec:representation_learning} to learn representations for RNNs from Formal Languages and the Sequential MNIST dataset. The second phase is dedicated to predicting properties of the RNNs. These predictions are either based on representations learned in the first phase or derived from fully supervised models trained from scratch.
We conduct the main experiments using 15 different random seeds for each model. The outcomes are presented as bootstrapped means with 95\% confidence intervals.
For RNNs trained on the Formal Languages dataset, performance is measured by the proportion of correctly generated strings (i.e., strings belonging to the language on which it was trained). For the Sequential MNIST dataset, we assess performance using standard digit classification validation accuracy.
For the Flattened Weights and the Parameter Transformer encoders, the training data is augmented by randomly permuting the neurons of the input RNN.

\subsection{Representation Learning}
\label{sec:representation_learning}

\begin{table}[]
\caption{Self-supervised validation losses.}    
\label{tab:pretraining_validation_results}
   \begin{center}
    \begin{small}
    \resizebox{\columnwidth}{!}{%
    \setlength{\tabcolsep}{2pt}
    \begin{tabular}{lcc}
    \toprule
    \textbf{Encoder} & \textbf{Formal Languages} & \textbf{Sequential MNIST}  \\
    \midrule

    Layer-Wise Statistics & 0.051 $\scriptstyle{ (0.050, 0.053) }$ & 0.039 $\scriptstyle{ (0.038, 0.039) }$ \\
    Flattened Weights & 0.045 $\scriptstyle{ (0.045, 0.046) }$ & 0.024 $\scriptstyle{ (0.024, 0.024) }$ \\
    Parameter Transformer & 0.043 $\scriptstyle{ (0.042, 0.044) }$ & 0.067 $\scriptstyle{ (0.067, 0.067) }$ \\
    DWSNet & 0.046 $\scriptstyle{ (0.046, 0.046) }$ & 0.024 $\scriptstyle{ (0.023, 0.025) }$ \\
    Non-Interactive Probing & 0.023 $\scriptstyle{ (0.019, 0.029) }$ & \textbf{0.017} $\scriptstyle{ (0.016, 0.017) }$ \\
    Interactive Probing & \textbf{0.015} $\scriptstyle{ (0.008, 0.022) }$ & \textbf{0.017} $\scriptstyle{ (0.017, 0.018) }$ \\
    \bottomrule
    \end{tabular}}
    \end{small}
    \end{center}
\end{table}

\begin{figure}[h]
    \centering
    \includegraphics[width=0.75\linewidth]{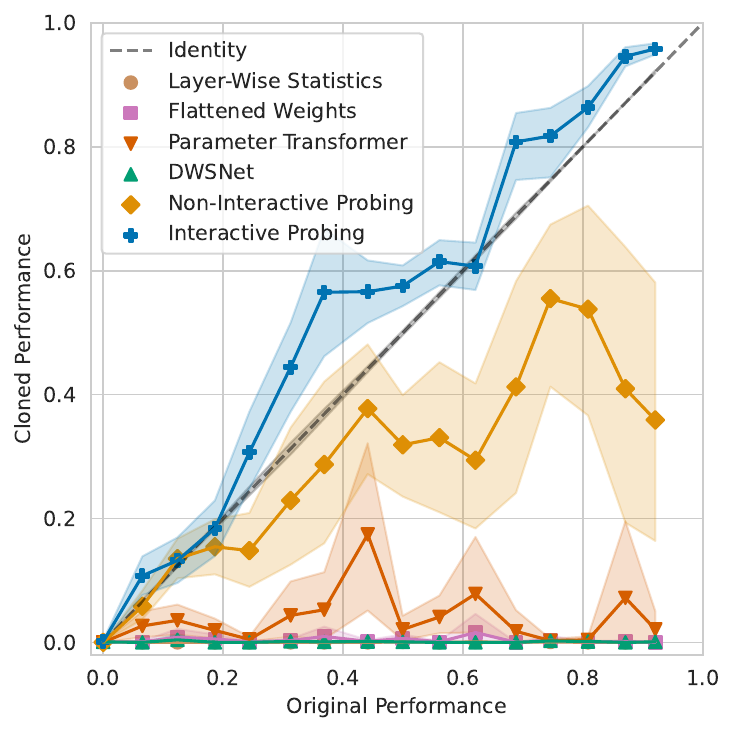}
    \vspace{-0.2cm}
    \caption{$f_\theta$'s original performance on formal language generation vs. the performance of $A_\xi$'s emulation based on $E_\phi(\theta)$ (validation data).}
    \label{fig:bach_emulation_performance}
\end{figure}

The six different encoder architectures are trained according to Objective~\ref{eq:loss}.
The hyperparameters of these encoders are selected to ensure a comparable number of parameters across all models.
Each encoder generates a 16-dimensional representation $z$.
An LSTM with two layers functions as the emulator $A_\xi$. 
The conditioning of $A_\xi$ on an RNN $f_\theta$ is implemented by incorporating a linear projection of the corresponding representation $z$ to the BOS token of the input sequence of $A_\xi$.
More details and hyperparameters can be found in Appendix~\ref{app:experiment_details}.

Table~\ref{tab:pretraining_validation_results} displays the emulation losses for the different encoder types on both the Formal Languages and the Sequential MNIST validation datasets. For Formal Languages, the interactive probing method outperforms the others. In the case of Sequential MNIST, the performance of both probing encoders is quite similar and superior to that of the mechanistic encoders.
Figure~\ref{fig:bach_emulation_performance} shows the emulation effectiveness of various RNNs from the Formal Languages dataset. 
We compute 16 equally spaced target performance values between the best and the worst performances in the datasets. 
For each target performance, we select the 15 RNNs with performances closest to each target. 
The x-position of each point represents the mean of the original performances, and the y-positions represent the mean of the emulated performances (the shaded areas give the 95\% confidence intervals).
The variance of the original performances for each point is relatively low, as can be seen from the shaded area around the identity line.
For this dataset, only the interactive probing encoder yields representations that enable $A_\xi$ to effectively emulate the original RNN.
Figure~\ref{fig:bach_original_vs_cloned} in the Appendix shows the analogous results also for the Formal Languages training and test set, Figure~\ref{fig:mnist_original_vs_cloned} for the Sequential MNIST datasets.

Figure~\ref{fig:embedding_space_bach} (top) examines the structure of the embedding space created by the interactive probing encoder for the Formal Languages validation set. 
It visually demonstrates that our method successfully learns coherent representation spaces of RNN weights.
Each point represents an RNN, reduced from $Z$ to two dimensions using principal component analysis. 
Each language forms its own cluster, with each cluster exhibiting a gradient representing the generation accuracy of the RNNs.
We highlight one cluster, corresponding to the language $L_{-2, 2, -2}$.
In contrast, t-SNE dimensionality reduction (bottom) fails to represent such structures in RNN weights. 
In t-SNE reductions, the weights of the RNN models for a single language are scattered across the entire space.
Similarly, Figure~\ref{fig:embedding_space_mnist} shows the embedding space for the Sequential MNIST dataset, where the task involves rotating MNIST digits. 
This rotation, a continuous scalar, is evident in the embedding space, alongside the classification accuracy of each model.  
In the t-SNE plot, although small clusters are noticeable for the nine snapshots of each run, the overall embedding space lacks a coherent structure. 
Visualizations of embedding spaces for all encoder architectures and OOD test sets are presented in Appendix~\ref{app:embedding_viz}.

\begin{figure*}[t]
    \centering
    \begin{minipage}[b]{0.48\linewidth}
        \centering
        \includegraphics[width=1.\linewidth]{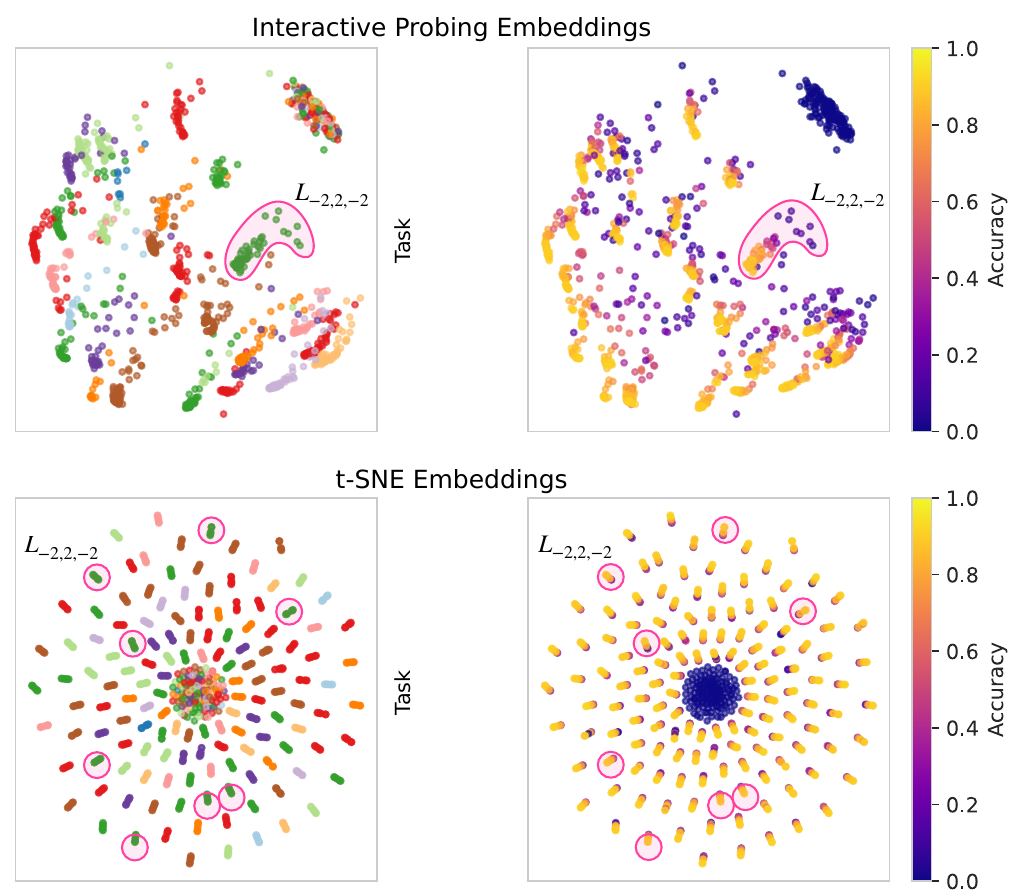}
        \caption{Visualization of embeddings of RNNs from the Formal Languages validation set.}
        \label{fig:embedding_space_bach}
    \end{minipage}\hfill
    \begin{minipage}[b]{0.48\linewidth}
        \centering
        \includegraphics[width=1.\linewidth]{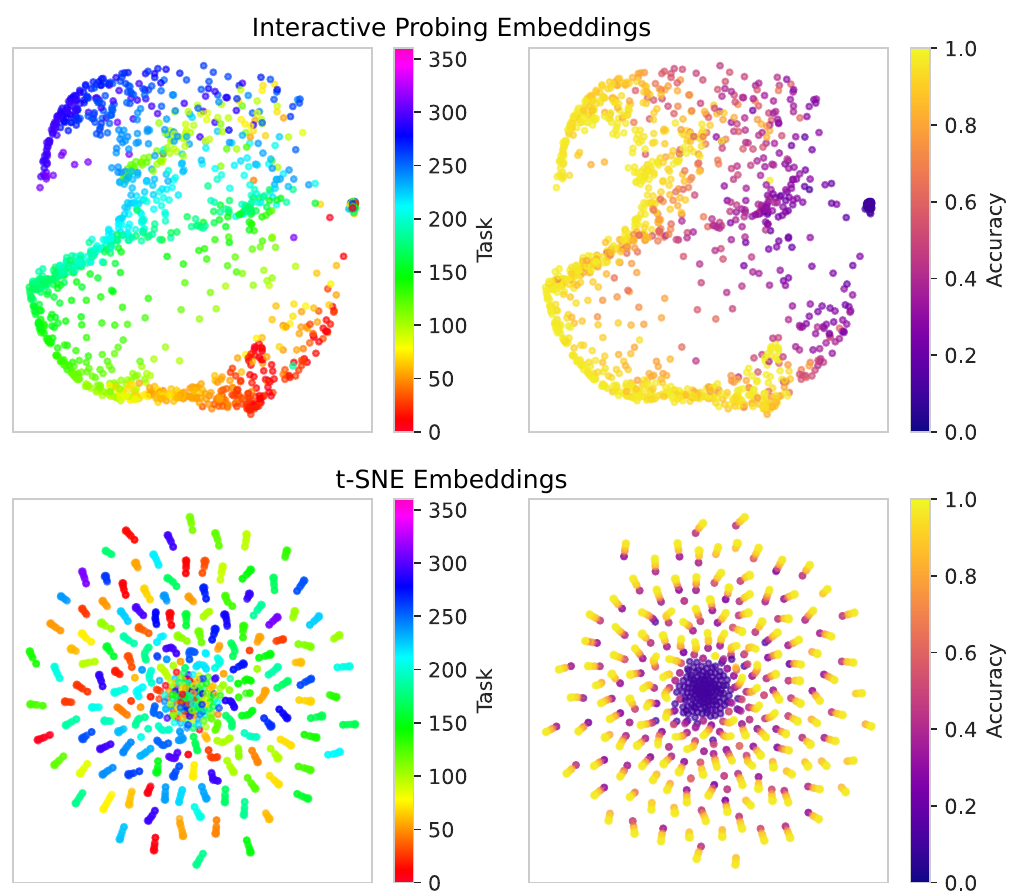}
        \caption{Visualization of embeddings of RNNs from the Sequential MNIST validation set.}
        \label{fig:embedding_space_mnist}
    \end{minipage}
\end{figure*}

\subsection{Downstream Property Prediction}
\label{sec:downstream_prediction}

To assess the effectiveness of representations learned through self-supervised learning, we evaluate them in predicting various properties of RNNs.
We train an MLP as a supervised prediction model using these representations, obtained from a fixed, pre-trained encoder $E_\phi$.
The properties to be predicted are stored as metadata for each RNN within the datasets.
For RNNs from the Formal Languages dataset, the properties are task and accuracy.
For RNNs from the Sequential MNIST dataset, the properties are task (i.e., the rotation of the digits), accuracy, training step, and generalization gap.
The formal language task is represented as a three-hot vector (for values $m_b$, $m_c$, and $m_d$), and the predictor is trained using binary cross-entropy.
For all scalar properties like accuracy, Sequential MNIST task, and generalization gap, the predictor is trained using mean squared error loss.
The Sequential MNIST training step prediction is framed as a 9-way classification problem, using cross-entropy loss.

For comparison, we also trained models for property prediction in a purely supervised manner. The setup is the same as for the pre-trained scenario, consisting of an RNN weight encoder followed by an MLP predictor. However, in this case, the encoder is not pre-trained but is randomly initialized and trained end-to-end with the predictor. Our primary interest lies in the generalization capabilities of the pre-trained $E_\phi$, rather than the predictor itself. Consequently, the training data for the predictor in both supervised and pre-trained settings includes half of the OOD test set.
We place particular emphasis on \textit{task} prediction (specifically, the type of formal language or the rotation of the MNIST dataset). This is because it demands a genuine understanding of the RNN's function, which goes beyond mere high-level statistical analysis of the weights.
Figure~\ref{fig:downstream_results_task_classification} (top) demonstrates that for task prediction in the Formal Languages dataset (OOD split), the pre-trained representations from an interactive probing encoder significantly outperform those from other encoders, as well as purely supervised models.
For the Sequential MNIST dataset (Figure~\ref{fig:downstream_results_task_classification}, bottom), both types of probing encoders surpass other architectures. In the supervised scenario, the non-interactive probing encoder excels, but the interactive version does not perform as well. 
We hypothesize that this is due to the limited information provided by supervised training compared to self-supervised pre-training, which does not offer sufficient feedback for the interactive encoder to develop effective probing sequences. Complete results for downstream predictions, both supervised and pre-trained, on validation and OOD test data are shown in Figures~\ref{fig:downstream_bar_chart_bach} and \ref{fig:downstream_bar_chart_mnist} in the Appendix.
Although these results do not definitively favor any single encoder architecture, they can be summarized as follows:
for the Formal Languages dataset, on both task and performance prediction, interactive probing generally yields the best results, particularly in the pre-trained setting. For the Sequential MNIST dataset, non-interactive probing performs best in the supervised setting, while both probing architectures are effective when pre-trained. In all other property predictions (accuracy, generalization gap, training step), DWSNet performs most consistently.

\begin{figure}[h]
\centering
\includegraphics[width=0.85\linewidth]{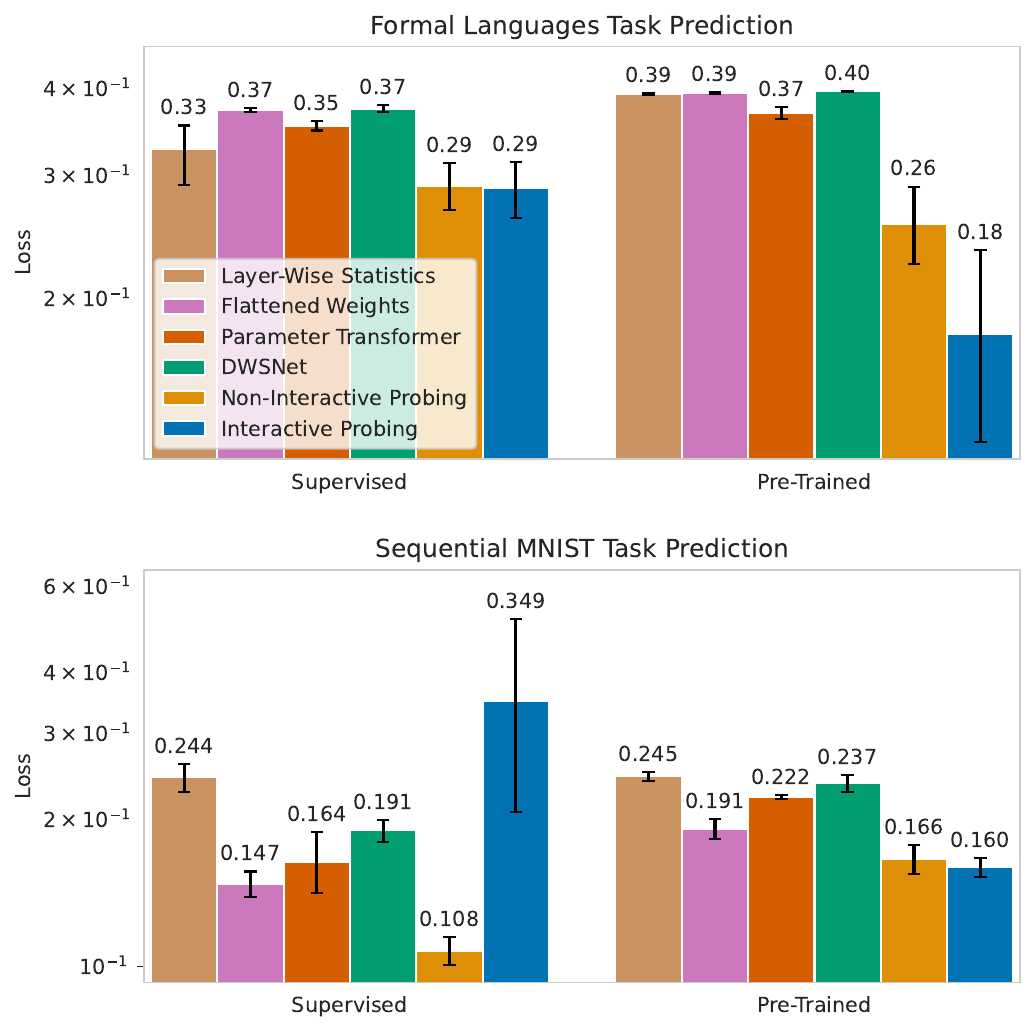}
\caption{Downstream performance of pre-trained and purely supervised model on task prediction for the Formal Languages and Sequential MNIST OOD test set.}
\label{fig:downstream_results_task_classification}
\end{figure}

\section{Discussion}
\label{discussion}

To effectively learn useful representations of RNN weights, three key components are necessary. First, datasets comprising a variety of trained RNNs are essential. The two datasets presented in this study, based on formal languages and tiled Sequential MNIST, are designed to support rapid experimental cycles and evaluations. Despite their scale, they offer sufficient complexity and challenge to analyze various approaches to the representation learning problem. These datasets address distinct facets of potential RNN applications: generative modeling of formal languages, which suits RNNs due to its precise, algorithmic nature, and digit classification, a less precise perceptual task.

Second, a pre-training method is required. We suggest a method rooted in the understanding that only those representations of RNN weights that can accurately emulate the original RNN possess meaningful information about the network. However, this is a necessary but not sufficient condition for effective representations. If the RNNs in the dataset used to train the encoder differ only superficially (and not algorithmically), then the derived representations would only reflect these superficial properties. 
Therefore, the task family from the training data must be complex enough to necessitate extracting complex information from the RNNs, but still be learnable by some method.
Our experiments demonstrate that only an interactive probing encoder can adequately capture the formal language task family. The Sequential MNIST task family (digit rotation) is simpler to learn, although the embedding visualizations suggest that Layer-wise Statistics and Parameter Transformer approaches fall short (see Figure~\ref{fig:embedding_space_viz_mnist_task_return}).

Third, an encoder architecture capable of handling the complex structure of RNN weights as input is crucial.
Except for Layer-Wise statistics, mechanistic encoders can approximate any function within the weight space.
Nonetheless, in practice, functionalist approaches outperform them both in self-supervised training objectives and task prediction. 
Our findings confirm the theoretical prediction that interactive probing is more efficient for certain datasets (Proposition~\ref{prop:fewer_interactions}), as evidenced by the results on the Formal Languages datasets.
However, interactive probing may suffer from training stability issues in some instances. 
This issue might be mitigated in the future through different training methods, loss functions, or regularizations.
When the interactive property is unnecessary or too complex to train, such as with the Sequential MNIST dataset, non-interactive probing may yield better results. 
Additionally, functionalist encoders offer the advantage of being agnostic to the precise architecture of $f_\theta$.
In more straightforward tasks like supervised property prediction, mechanistic encoders like or Layer-Wise Statics or DWSNet excel.

To date, our work is limited to relatively small RNNs. 
All approaches presented here can in principle be adapted in a straightforward way to other architectures, such as state space models~\cite{gu2021efficiently}, deep equilibrium models~\cite{bai2019deep}, or transformers~\cite{vaswani2017attention}.
For scaling to larger models, we see the greatest potential in functionalist probing approaches. 
These are, in principle, agnostic to the architecture and size of the RNN (or any general sequence model), provided it is differentiable. 
By employing policy gradient methods, we could even process models that are impossible or prohibitively expensive to differentiate.
\looseness=-1

\section{Conclusion and Future Work}
We have proposed a framework that uses self-supervised learning to derive useful representations of RNN weights. 
Through this framework, we have trained and evaluated various weight encoder architectures.  
Notably, our newly proposed interactive probing approach is the only method capable of learning suitable representations for formal language tasks.
This finding corroborates our theoretical results, demonstrating that interactive probing can, in certain situations, outperform non-interactive probing.

This work establishes a foundation for numerous future applications. For instance, the techniques introduced here can be applied in the context of reinforcement learning within partially observable environments, where they can facilitate policy representation and improvement (see e.g., \citet{raileanu2020fast, faccio2020parameter, faccio2022general}), as well as exploration and skill discovery \cite{herrmann2022learning}.
Additionally, RNN weight representations could be beneficial in meta-learning and few-shot learning scenarios. 
Another field where our work might be useful is the mechanistic interpretability of sequence models (e.g., \cite{weiss2018extracting, olsson2022context}). 
Typically, it involves a lot of labor-intensive reverse engineering. Training models to create meaningful representations of RNN (or general sequence model) weights may assist or even partially automate these efforts. 
Our functionalist probing approaches have the potential to extract information not only from populations of small models but also from large foundational models. 
This can be achieved by conditioning either the encoder or the foundational model itself on a specific task, a concept that aligns with discussions in previous research \cite{learningtothink2015, zeng2022socratic, zhuge2023mindstorms}.

\section*{Acknowledgements}
We thank Louis Kirsch, Piotr Piekos, Aditya Ramesh, and Wenyi Wang for insightful discussions. 
This work was supported by the ERC Advanced Grant (no: 742870).
We also thank NVIDIA Corporation for donating a DGX-1 as part of the Pioneers of AI
Research Award.



\bibliography{main}
\bibliographystyle{icml2024}

\newpage
\appendix
\onecolumn

\section{Theoretical results}
\label{app:proofs}

\begin{figure}[t]
    \centering
    \includegraphics[width=0.3\linewidth]{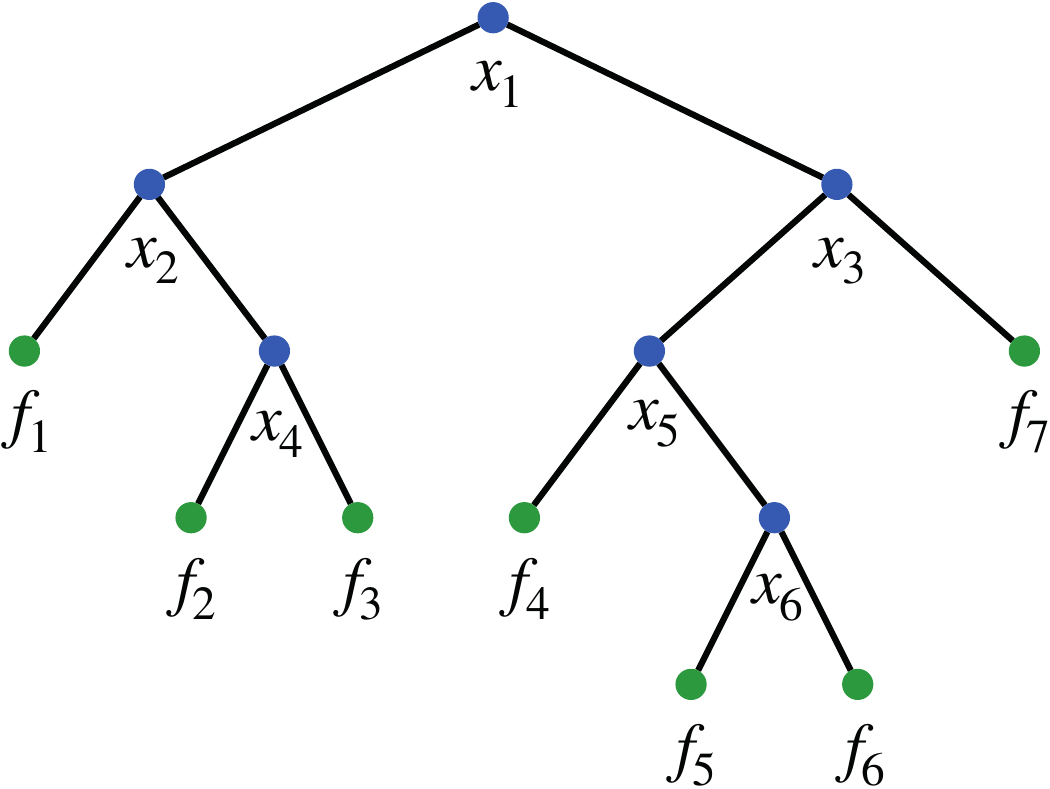}
    \caption{A binary tree constructed as described in the proof of Theorem~\ref{th:interaction_limit}. 
    Giving the inputs $x_j$ corresponding to all branching nodes to a function allows to uniquely identify it.}
    \label{fig:binary_tree}
\end{figure}

\begin{lemma}
\label{lemma:split}
Given any subset $G \subseteq D, |G| \geq 2$, there exists an input $x$ that can be computed for which $f_a(x) \neq f_b(x)$ with $f_a, f_b \in G$.
\end{lemma}

\begin{proof}
This follows immediately from the fact that all functions in $G$ are total computable and functionally distinct.
\end{proof}

\ab*
\begin{proof}
According to Lemma~\ref{lemma:split}, it is possible to split any set $G \subseteq D, |G| \geq 2$ into two nonempty, non-overlapping subsets: $G_a := \{f \in G | f(x_j) = f_a(x_j)\}$ and $G_b := \{f \in G | f(x_j) \neq f_a(x_j)\}$ for some $f_a \in G$ and $x_j \in \mathbb{N}$. 
Any resulting subset that has at least two members can be split again using the same procedure with a different probing input $x_{j+1}$.
Starting from the full set $D$, it is possible to construct a binary tree (see Figure~\ref{fig:binary_tree}) where the leaves are subsets of $D$ containing exactly one uniquely identified function.
The branching (i.e., non-leaf) nodes correspond to the splitting operation, which involves observing the output of a specific probing input $x_j$.

The Interrogator can identify a given function $f_C \in D$ by providing it with all inputs $x_j$ corresponding to the branching nodes in the binary tree and observing the outputs.
Since any binary tree with $n$ leaves has exactly $n-1$ branching nodes, any function $f_C \in D$ can be identified using $|D| - 1$ interactions.
\end{proof}

Of course, there are `easy' function sets in the sense that their members
can be identified using much fewer interactions.
Consider, for example, the set $\{n \mapsto i\ \forall n | 1 \leq i \leq L\}$.
Here, only one (any) probing input is necessary, since the identity of the function can be directly read from the output.

\cd*
\begin{proof}
It is easy to construct function sets $D$ for which the members cannot be identified in less than $|D|-1$ interactions, even by an interactive Interrogator.

One such function set is $\{\xi_i | 1 \leq i \leq L \}$ with 
$\xi_i: n \mapsto 
    \begin{cases}
    0 \text{ if } n=i,\\
    n \text{ else}
    \end{cases}.$
In the worst case, there is no way around trying all inputs $1, \dots, L-1$.
\end{proof}

\ef*

\begin{proof}

We construct a concrete set of functions that an interactive Interrogator can identify exponentially faster than a non-interactive one.
Consider the family of context-sensitive languages
\begin{equation}
\label{eq:bach_language}
L_{m_2, m_3, \dots, m_k} := \{a_1^{n} a_2^{n+m_2} a_3^{n+m_3} \dots a_k^{n+m_k} | n \in \mathbb{N}\},
\end{equation}
 with $m_2, \dots, m_k \in \mathbb{N}$ and $a_1, \dots, a_k$ being the letters or tokens of the language.
The parameters $m_i$ define the relative number of times different tokens may appear.
As an example, one member of the language $L_{2, 1}$ is the string $a_1 a_2 a_2 a_2 a_3 a_3$.

Let $G_L := \{L_{m_2, \dots, m_k} | m_2, \dots, m_k \in \{1, \dots, M\}\}$, i.e., a set of such languages with different parameters $m_i$.
$G_L$ contains $M ^ {(k-1)}$ languages.
To each language $L_{m_2, \dots, m_k}$, we can assign a unique generative function $g_{m_2, \dots, m_k}$.
This function, given a partial string from the language, returns a list of the allowed tokens for the next step.
If the input string is not a prefix of string from the language, it returns the empty string $\epsilon$.
For example, $g_{2, 1}(a_1) = (a_1, a_2)$, $g_{2,1}(a_1 a_2 a_2) = (a_2)$, and $g_{2,1}(a_1 a_2 a_2 a_3) = \epsilon$.
Our function set $D_L$ is a set of such generative functions, $D_L := \{g_{m_2, \dots, m_k} | m_2, \dots, m_k \in \{1, \dots, M\}\}$.

For an interactive Interrogator, there is a simple strategy to identify a given function $g_C \in D_L$ using $M \cdot (k-1)$ interactions:
The first input is the string $a_1 a_2$.
From there on, the Interrogator acts as an autoregressive generative model---it appends the allowed token returned by $g_C$ to the string and uses it as the new input.
Only one valid token will be returned by $g_C$ for all probing input strings that are generated using this approach since the $n$ is determined to be $1$ from the first input string.
This is repeated until $\epsilon$ is returned, which is after a maximum of $M \cdot (k-1)$ calls to $g_C$.
The last probing input string will be of the form $a_1 a_2^{r_2} \dots a_k^{r_k}$, from which the language can easily be inferred to be $L_{r_2-1, \dots, r_k-1}$.

The non-interactive Interrogator cannot use this strategy, since every probing input except the first depends on $g_C$'s output for the previous probing input. 
We can show that in the non-interactive setting, exponentially many calls to $g_C$ are needed to identify it.
Assuming $n=1$, there are $M ^ {k-2}$ unique prefixes for the first token $a_k$. 
Each of these prefixes is only allowed in $M$ languages $L_{m_2, \dots, m_{k-1}, \cdot}$, namely the ones with fixed $m_1, \dots, m_{k-1}$.
Remember that $g_C$ returns $\epsilon$ whenever it is given a substring that is not part of its language. 
That means, to determine $m_k$, $M ^ {k-2} (M-1)$ different inputs have to be given to $f_C$. 
Only then it is guaranteed that the only informative string about the unknown value of $m_k$, namely $a_1 a_2^{m_2} \dots a_k^{m_k}$, is among the probing inputs.
It follows that to determine all values $m_2 \dots, m_k$ and identify the exact language of $g_C$, a total of $\sum_{b=2}^{k-2} M^b (M-1) = M^{k-1} - M^2$ inputs are needed.

In short, to identify a function from the set $D_L$ described above, an interactive Interrogator needs $O(Mk)$ probing inputs, whereas a non-interactive one needs $O(M^k)$.
\end{proof}

\section{DWSNet Details}
\label{app:dwsnet}

\begin{figure}[h]
\centering
\includegraphics[width=0.7\linewidth]{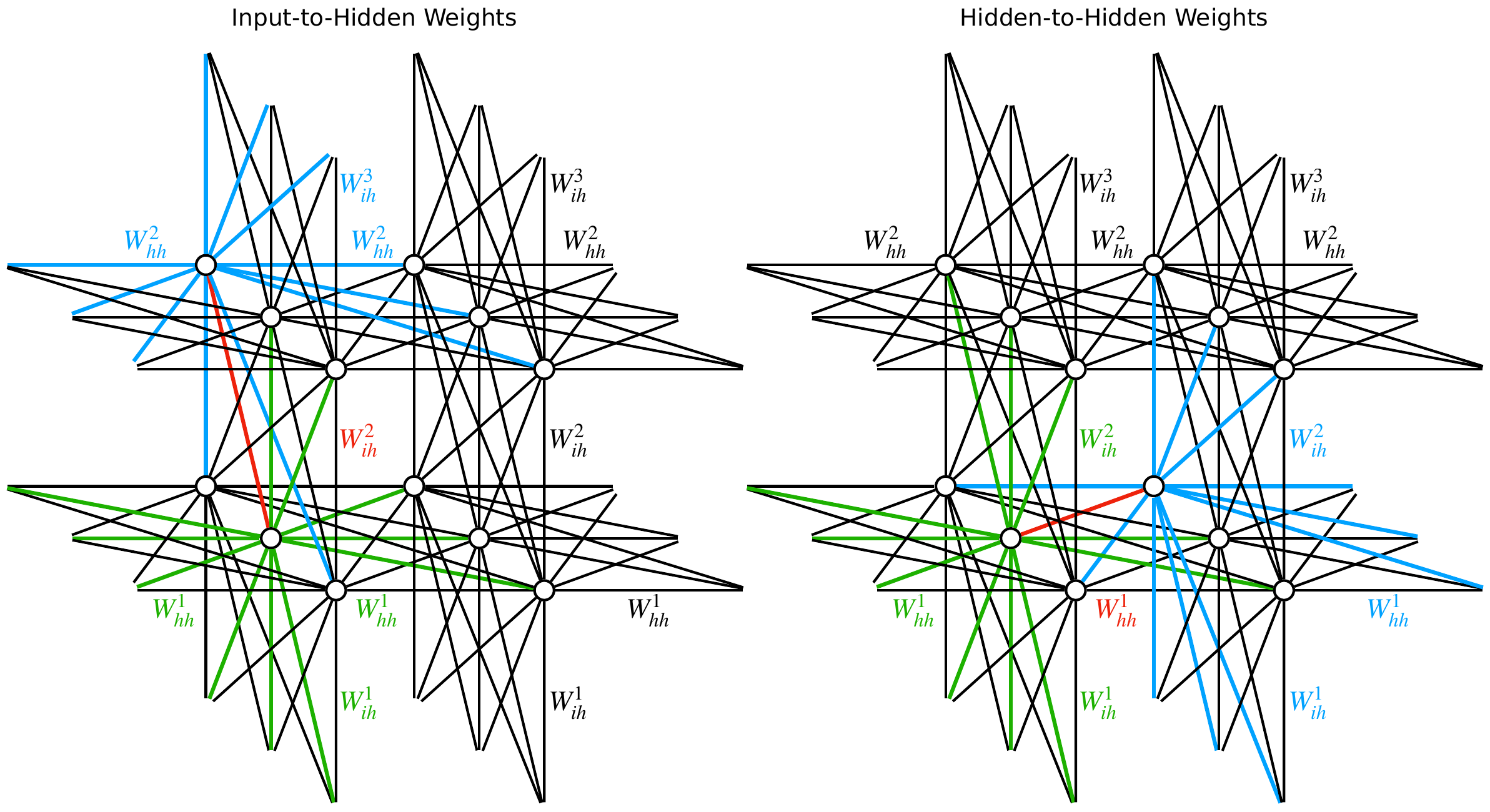}
\caption{All connections from and to a single weight of an RNN in an input-to-hidden weight matrix (left) or hidden-to-hidden weight matrix (right). They must all be taken into account to ensure the equivariance of DWSNets for RNN weights.}
\label{fig:dwsnet_equivariance}
\end{figure}

The implementation of our permutation equivariant weight space layer for RNNs is based on the work of \citet{zhou2023permutation}.
Unlike feedforward networks, RNNs feature two distinct types of weight matrices: those mapping inputs to hidden states ($W^l_{ii}$ for layer $l$) and those mapping hidden states to other hidden states ($W^l_{hh}$).
To ensure the equivariance of weight space feature maps, it is crucial to consider all incoming and outgoing connections for both types of matrices (as illustrated in Figure~\ref{fig:dwsnet_equivariance}).
We define an equivariant linear function $H_i$ for the input to hidden weights as:

\begin{align*}
    H_i(W_{ih})^l_{jk} = & \bigg( \sum_s a_{ih}^{l, s} (W_{ih}^l)_{*, *} + b_{ih}^{l, s} (W_{hh}^l)_{*, *} \bigg) \\
    &+ c_{ih}^{l,l} (W_{ih}^l)_{*, k} + c_{ih}^{l,l-1} (W_{ih}^{l-1})_{k, *} + d_{ih}^{l,l} (W_{ih}^l)_{j, *} + d_{ih}^{l,l+1} (W_{ih}^{l+1})_{*, j} \\
    &+ e_{ih}^{l, l} (W_{hh})^l_{*, k} + e_{ih}^{l, l-1} (W_{hh})^{l-1}_{k, *} + f_{ih}^{l,l} (W_{hh}^l)_{j, *} + f_{ih}^{l,l-1} (W_{hh}^{l-1})_{*, j} \\
    &+ g_{ih}^l (W_{ih}^l)_{jk}.
\end{align*}

The learnable parameters are $a_{ih}, b_{ih}, c_{ih}, d_{ih}, e_{ih}, f_{ih}$ and $g_{ih}$. 
For hidden to hidden weights, the equivariant linear function $H_h$ is:

\begin{align*}
    H_h(W_{hh})^l_{jk} = & \bigg( \sum_s a_{hh}^{l, s} (W_{ih}^l)_{*, *} + b_{hh}^{l, s} (W_{hh}^l)_{*, *} \bigg) \\
    &+ c_{hh}^{l,l} (W_{hh}^l)_{*, k} + c_{hh}^{l,l-1} (W_{hh}^{l-1})_{k, *} + d_{hh}^{l,l} (W_{hh}^l)_{j, *} + d_{hh}^{l,l+1} (W_{ih}^{l+1})_{*, j} \\
    &+ e_{hh}^{l, l} (W_{ih})^l_{*, k} + e_{hh}^{l, l+1} (W_{ih})^{l+1}_{k, *} + f_{hh}^{l,l} (W_{hh}^l)_{j, *} + f_{hh}^{l,l+1} (W_{hh}^{l+1})_{*, j} \\
    &+ g_{hh}^l (W_{hh}^l)_{jk},
\end{align*}

with learnable parameters $a_{hh}, b_{hh}, c_{hh}, d_{hh}, e_{hh}, f_{hh}$ and $g_{hh}$.
In LSTMs, the weight matrices for the input, output, forget, and cell gates are treated as four channels of the weight space input feature map.

Given that an RNN is characterized by the mapping from input and hidden states to output and updated hidden states by a single RNN cell, the universality proofs by \citet{navon2023equivariant} are applicable.
Consequently, under similar regularity conditions, our adapted DWSNets can approximate any function within the weight space.

\section{Dataset Details}

The process for creating the Formal Languages and Sequential MNIST RNN weight datasets is consistent across all instances. In each of the 1000 training iterations, a new LSTM network is initialized, and a task is chosen at random. This model undergoes training for a total of 20000 steps using the AdamW optimizer \cite{loshchilov2017decoupled}, with a weight decay of $10^{-4}$ and a batch size of 32.
We implement a piece-wise linear learning rate schedule that adjusts the learning rate between $[0.01, 0.003, 0.0003]$ at steps $[0, 10000, 20000]$.
The model weights are saved as a datapoint at the training steps $[0, 100, 200, 500, 1000, 2000, 5000, 10000, 20000]$, along with 100 rollout sequences and various attributes (such as the task, training step, and performance metrics, including Sequential MNIST validation and training loss).

\section{Experiment Details}
\label{app:experiment_details}

All experiments, both pre-training and downstream property prediction tasks, are run for 100k training steps with an early stopping criterion based on validation loss performance. 
All MLPs use ReLU activation functions within their hidden layers.
The hyperparameters for the self-supervised pre-training phase are detailed in Table~\ref{tab:pretraining_hyperparams}, the additional settings for the prediction phase are listed in Table~\ref{tab:downstream_hyperparams}.
Tables~\ref{tab:layerwise_hyperparams}-\ref{tab:interactive_hyperparams} report the hyperparameters for six distinct RNN weight encoder architectures, in addition to specifying the total encoder model sizes for both the Formal Languages and Sequential MNIST studies. The probing encoders use a probing sequence length of 22 for Formal Languages and 51 for Sequential MNIST.

\begin{table}[h]
    \begin{minipage}[t]{0.45\linewidth} 
    \centering
    \caption{Pre-Training Hyperparameters}
    \label{tab:pretraining_hyperparams}
    \begin{tabular}{ll}
    \toprule
    \textbf{Hyperparameter} & \textbf{Value} \\ 
    \midrule
    $A$ Hidden Size & 256 \\
    $A$ \#Layers & 2 \\
    $z$ Size & 16 \\
    Batch Size & 64 \\
    Optimizer & AdamW \\
    Learning Rate & 0.0001 \\
    Weight Decay & 0.01 \\
    Gradient Clipping & 0.1\\
    \bottomrule
    \end{tabular}
    \end{minipage}
    \hfill
    \begin{minipage}[t]{0.45\linewidth}
    \centering
    \caption{Additional Downstream Hyperparameters}
    \label{tab:downstream_hyperparams}
    \begin{tabular}{ll}
    \toprule
    \textbf{Hyperparameter} & \textbf{Value} \\ 
    \midrule
    Predictor MLP Size & 128 \\
    Predictor MLP \#Layers & 2 \\
    \bottomrule
    \end{tabular}
    \end{minipage}
\end{table}

\begin{table}[h]
    \begin{minipage}[t]{0.45\linewidth} 
    \centering
    \caption{Layer-Wise Statistics Hyperparameters}
    \label{tab:layerwise_hyperparams}
    \begin{tabular}{ll}
    \toprule
    \textbf{Hyperparameter} & \textbf{Value} \\ 
    \midrule
    MLP Hidden Size & 768 \\
    MLP \#Layers & 3 \\
    \textit{\#Parameters Formal Languages} & \textit{1248016} \\
    \textit{\#Parameters Sequential MNIST} & \textit{1248016} \\
    \bottomrule
    \end{tabular}
    \end{minipage}
    \hfill
    \begin{minipage}[t]{0.45\linewidth} 
    \centering
    \caption{Flattened Weights Hyperparameters}
    \label{tab:flattened_hyperparams}
    \begin{tabular}{ll}
    \toprule
    \textbf{Hyperparameter} & \textbf{Value} \\ 
    \midrule
    MLP Hidden Size & 128 \\
    MLP \#Layers & 3 \\
    \textit{\#Parameters Formal Languages} & \textit{1797264} \\
    \textit{\#Parameters Sequential MNIST} & \textit{1978000} \\
    \bottomrule
    \end{tabular}
    \end{minipage}
\end{table}

\begin{table}[h]
    \begin{minipage}[t]{0.45\linewidth} 
    \centering
    \caption{Parameter Transformer Hyperparameters}
    \label{tab:transformer_hyperparams}
    \begin{tabular}{ll}
    \toprule
    \textbf{Hyperparameter} & \textbf{Value} \\ 
    \midrule
    Size & 128 \\
    Transformer MLP Hidden Size & 512 \\
    \#Layers & 6 \\
    \#Heads & 2 \\
    \textit{\#Parameters Formal Languages} & \textit{1276688} \\
    \textit{\#Parameters Sequential MNIST} & \textit{1278480} \\
    \bottomrule
    \end{tabular}
    \end{minipage}
    \hfill
    \begin{minipage}[t]{0.45\linewidth} 
    \centering
    \caption{DWSNet Hyperparameters}
    \label{tab:dwsnet_hyperparams}
    \begin{tabular}{ll}
    \toprule
    \textbf{Hyperparameter} & \textbf{Value} \\ 
    \midrule
    \# Channels & 48 \\
    \# Layers & 4 \\
    \textit{\#Parameters Formal Languages} & \textit{1279736} \\
    \textit{\#Parameters Sequential MNIST} & \textit{1282400} \\
    \bottomrule
    \end{tabular}
    \end{minipage}
\end{table}

\begin{table}[h]
    \begin{minipage}[t]{0.45\linewidth} 
    \centering
    \caption{Non-Interactive Probing Hyperparameters}
    \label{tab:noninteractive_hyperparams}
    \begin{tabular}{ll}
    \toprule
    \textbf{Hyperparameter} & \textbf{Value} \\ 
    \midrule
    $E_R$ Hidden Size & 256 \\
    $E_R$ \#Layers & 2 \\
    $E_I$ \& $E_O$ Hidden Size & 128 \\
    $E_I$ \& $E_O$ \#Layers & 1 \\
    \textit{\#Parameters Formal Languages} & \textit{1241718} \\
    \textit{\#Parameters Sequential MNIST} & \textit{1254016} \\
    \bottomrule
    \end{tabular}
    \end{minipage}
    \hfill
    \begin{minipage}[t]{0.45\linewidth} 
    \centering
    \caption{Interactive Probing Hyperparameters}
    \label{tab:interactive_hyperparams}
    \begin{tabular}{ll}
    \toprule
    \textbf{Hyperparameter} & \textbf{Value} \\ 
    \midrule
    $E_R$ Hidden Size & 256 \\
    $E_R$ \#Layers & 2 \\
    $E_I$ \& $E_O$ Hidden Size & 128 \\
    $E_I$ \& $E_O$ \#Layers & 1 \\
    \textit{\#Parameters Formal Languages} & \textit{1235830} \\
    \textit{\#Parameters Sequential MNIST} & \textit{1240960} \\
    \bottomrule
    \end{tabular}
    \end{minipage}
\end{table}

\FloatBarrier
\newpage

\section{Additional Results}
\label{app:additional_results}

Figure~\ref{fig:bach_original_vs_cloned} shows the emulation performance of different emulators $A_\xi$ on the Formal Languages datasets. 
We observe overfitting on the training data of the non-interactive probing and parameter transformer encoders. 
Note, however, that this occurs despite the use of early stopping, meaning the models evaluated are those with the lowest validation loss.

Figure~\ref{fig:mnist_original_vs_cloned} shows the emulation performance on the Sequential MNIST datasets.
There we can observe an interesting effect: in the Parameter Transformer setup, $A_\xi$ always classifies the digits with high accuracy, even when conditioned on low-performing RNNs.
We suggest the following explanation: from Figures~\ref{fig:downstream_bar_chart_mnist}, \ref{fig:embedding_space_viz_mnist_task_return} and \ref{fig:embedding_space_viz_mnist_step_generalization} it is clear that for Sequential MNIST, the parameter transformer fails to learn informative representations.
What does the emulator $A_\xi$ learn in this case? 
For Sequential MNIST, the target distributions $y$ typically assign a high probability to the correct digit. 
Therefore, the emulator is incentivized to learn a rotation-invariant MNIST classifier that will perform well on any input, regardless of the conditioning. 
However, there is an additional nuance: the reverse KL divergence loss function prefers high entropy output distributions from the emulator. 
This means that while the emulator's performance in terms of classification loss would be poor (since high entropy distributions generally result in high losses), its performance in terms of classification accuracy can still be good even with a high entropy distribution, as long as the correct digits have a probability that is at least slightly higher than the other ones.

\begin{figure}[H]
    \centering
    \begin{minipage}[b]{0.32\linewidth}
        \centering
        \includegraphics[width=1.\linewidth]{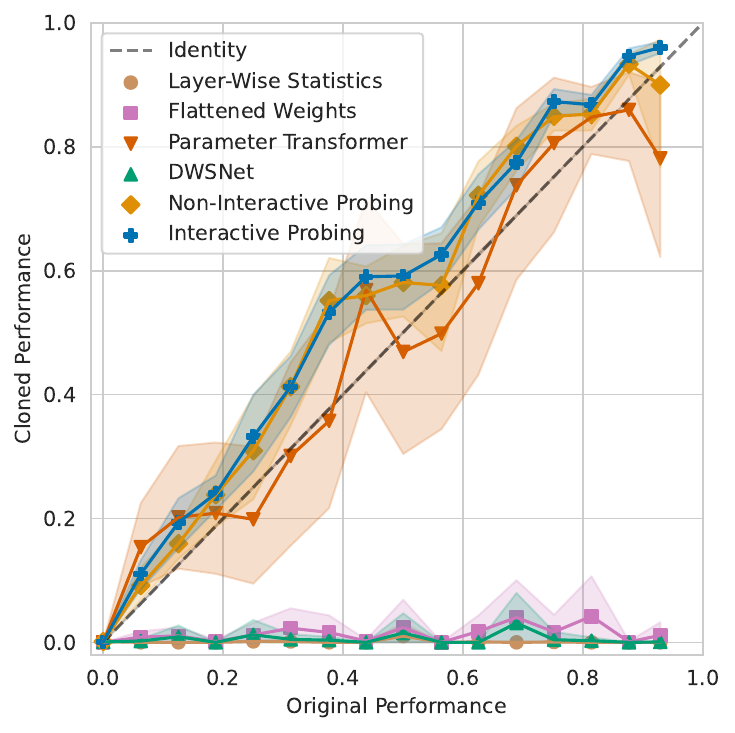}
    \end{minipage}\hfill
    \begin{minipage}[b]{0.32\linewidth}
        \centering
        \includegraphics[width=1.\linewidth]{figures/bach_eval_cloned_returns.pdf}
    \end{minipage}
    \begin{minipage}[b]{0.32\linewidth}
        \centering
        \includegraphics[width=1.\linewidth]{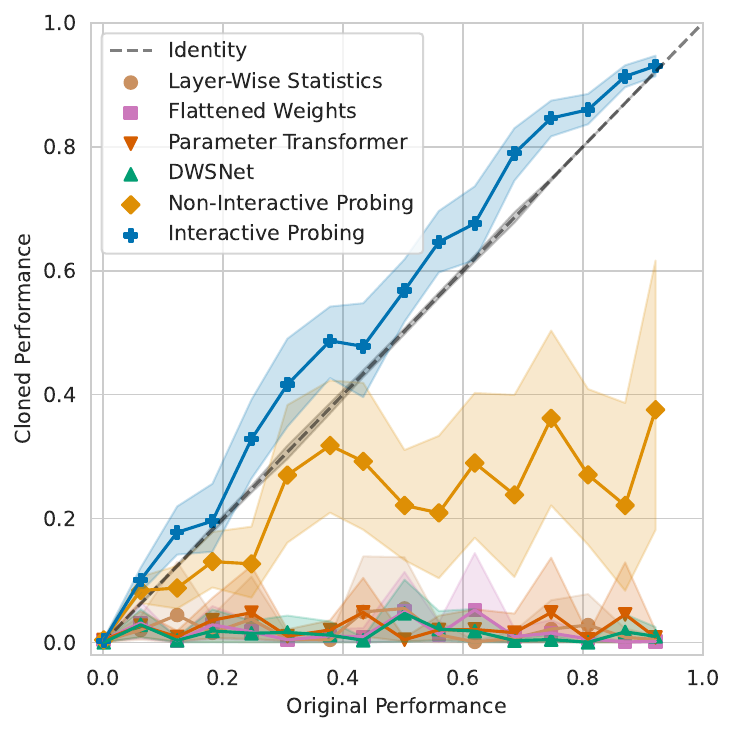}
    \end{minipage}
    \caption{Formal Languages original performance vs. emulated performance. Training (left), validation (middle), and OOD test data (right).}
    \label{fig:bach_original_vs_cloned}
\end{figure}

\begin{figure}[H]
    \centering
    \begin{minipage}[b]{0.32\linewidth}
        \centering
        \includegraphics[width=1.\linewidth]{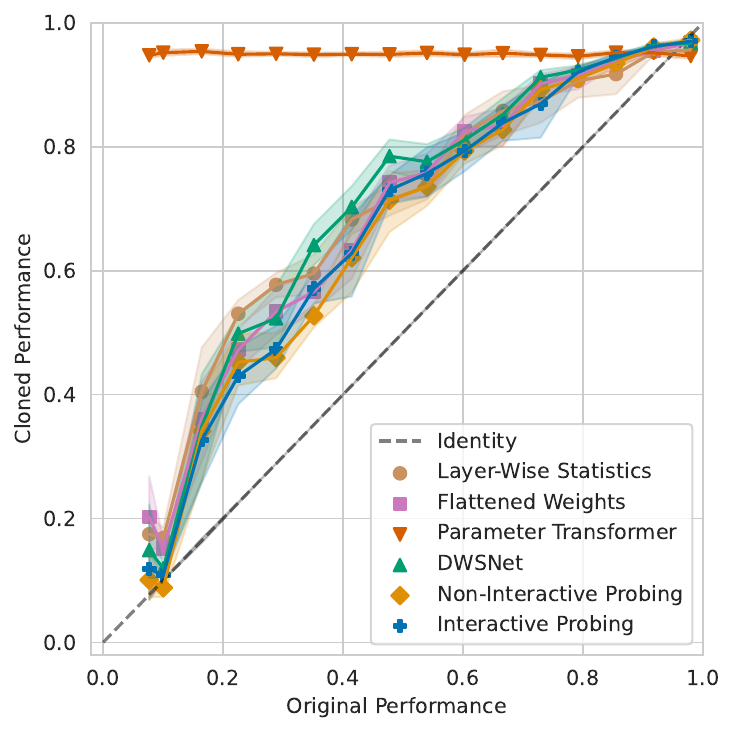}
    \end{minipage}\hfill
    \begin{minipage}[b]{0.32\linewidth}
        \centering
        \includegraphics[width=1.\linewidth]{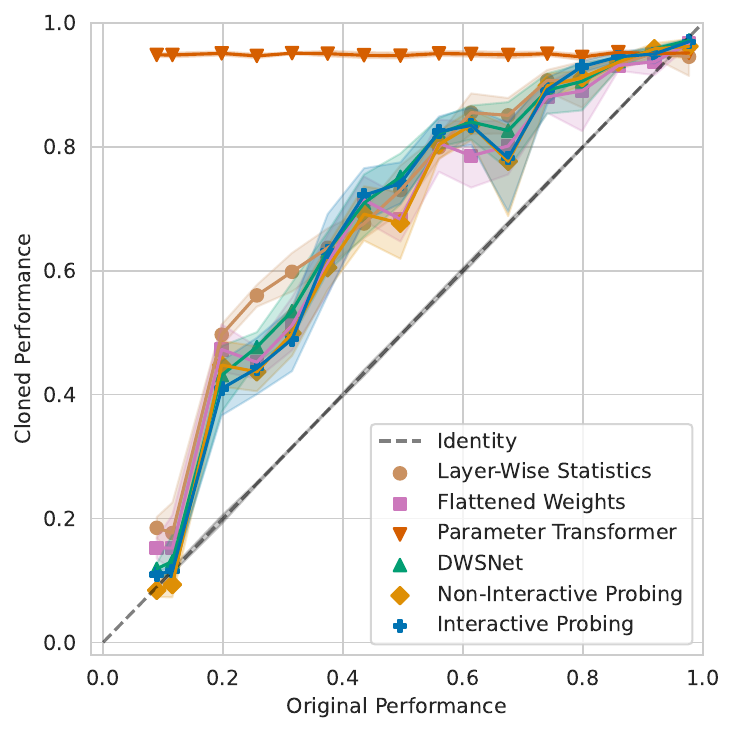}
    \end{minipage}
    \begin{minipage}[b]{0.32\linewidth}
        \centering
        \includegraphics[width=1.\linewidth]{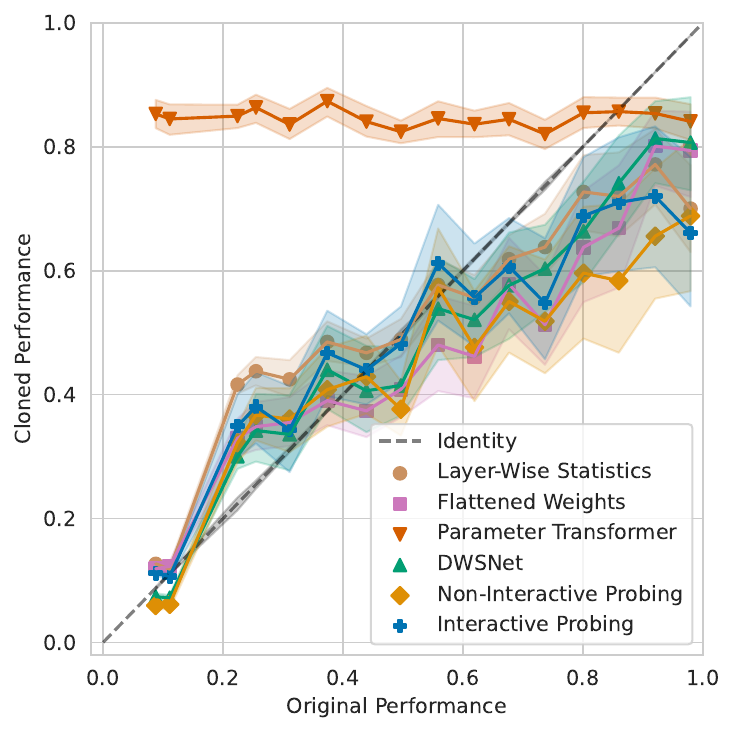}
    \end{minipage}
    \caption{Sequential MNIST original performance vs. emulated performance. Training (left), validation (middle), and OOD test data (right)}
    \label{fig:mnist_original_vs_cloned}
\end{figure}

\subsection{Probing Sequence Number and Length}

The probing encoders presented in this work can utilize a fixed number of parallel probing sequences.
If not specified otherwise, 8 probing sequences are used.
Figures~\ref{fig:bach_loss_vs_probing_seq} and \ref{fig:mnist_loss_vs_probing_seq} demonstrate that using more probing sequences correlates with lower loss. Interactive probing proves highly effective for RNNs applied to the Formal Languages dataset, though it faces challenges with training stability.
In principle, one probing sequence is sufficient to solve the task, as evidenced by the top-performing model among 15 seeds achieving very low loss regardless of the number of probing sequences. Nonetheless, training stability improves with an increase in probing sequences.

For RNNs analyzing the Sequential MNIST dataset, the optimal length for probing sequences is 51 (49 plus BOS and EOS), matching the number of tiles representing a full digit. In the case of Formal Languages, a probing sequence length of 22 is sufficient for interactive probing to identify all languages of the dataset.
Figure~\ref{fig:bach_loss_vs_probing_seq_len} illustrates that shorter probing sequences yield poorer results, while longer ones do not increase performance and can even harm training stability. Therefore, the length of the probing sequence should be as long as necessary but no longer.

Figure~\ref{fig:bach_probing_sequences} displays the probing sequences generated by different Interactive Probing encoders for a specific RNN that produces strings from languages $L_{-3,-1,1}$ (the encoders for this figure learn only one probing sequence). A length of $7$ is notably insufficient, failing to learn any meaningful probing sequence. However, at sequence lengths of $12$, $22$, and $42$, the encoder successfully learns to probe actual strings of varying lengths from $L_{-3,-1,1}$.
Note that while there is a string of length less than $12$ for this particular language, this does not apply to all languages in the dataset.

\begin{figure*}[h]
    \centering
    \begin{minipage}[t]{0.32\linewidth}
        \centering
        \includegraphics[width=1\linewidth]{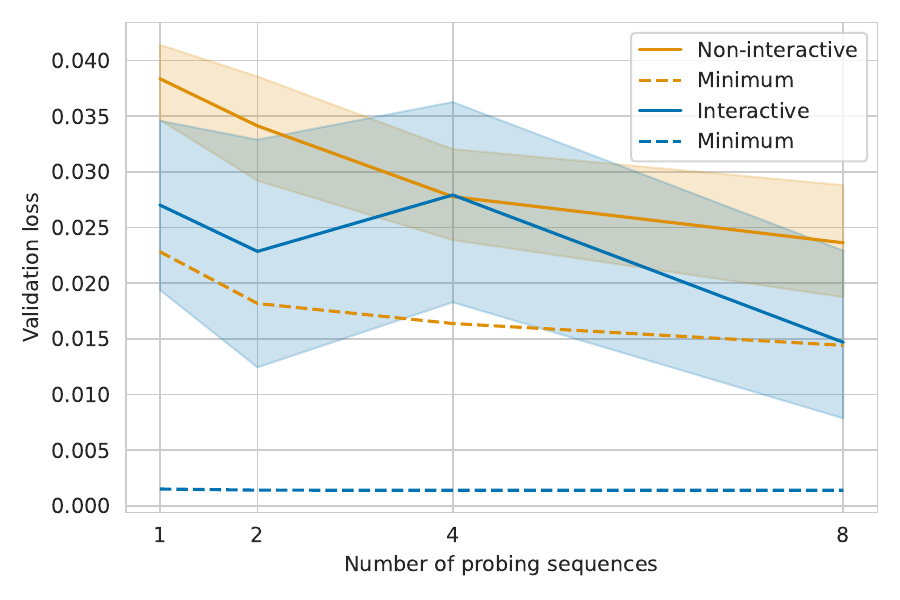}
        \caption{Formal Languages validation loss vs. number of probing sequences}
        \label{fig:bach_loss_vs_probing_seq}
    \end{minipage}\hfill
    \begin{minipage}[t]{0.32\linewidth}
        \centering
        \includegraphics[width=1\linewidth]{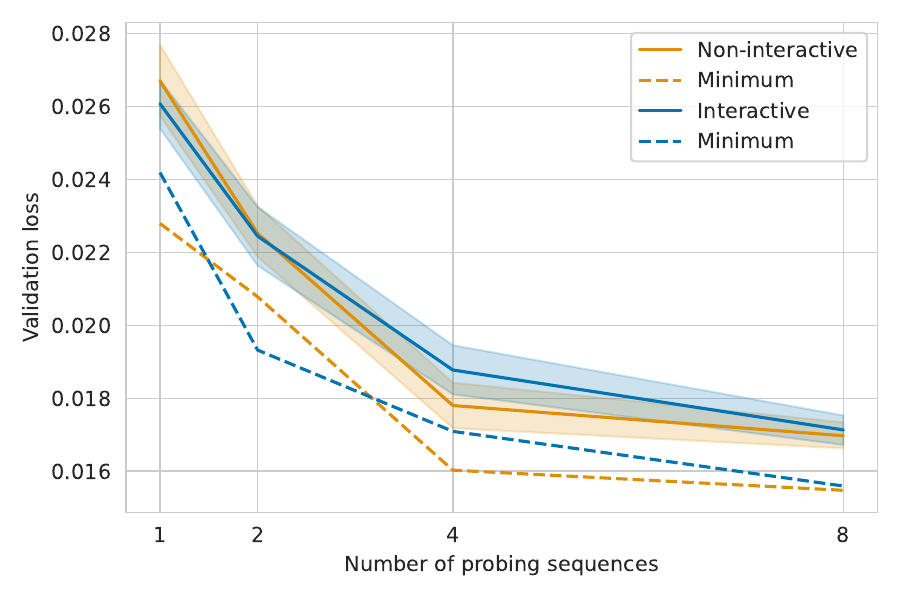}
        \caption{Sequential MNIST validation loss vs. number of probing sequences}
        \label{fig:mnist_loss_vs_probing_seq}
    \end{minipage}\hfill
    \begin{minipage}[t]{0.32\linewidth}
        \centering
        \includegraphics[width=1\linewidth]{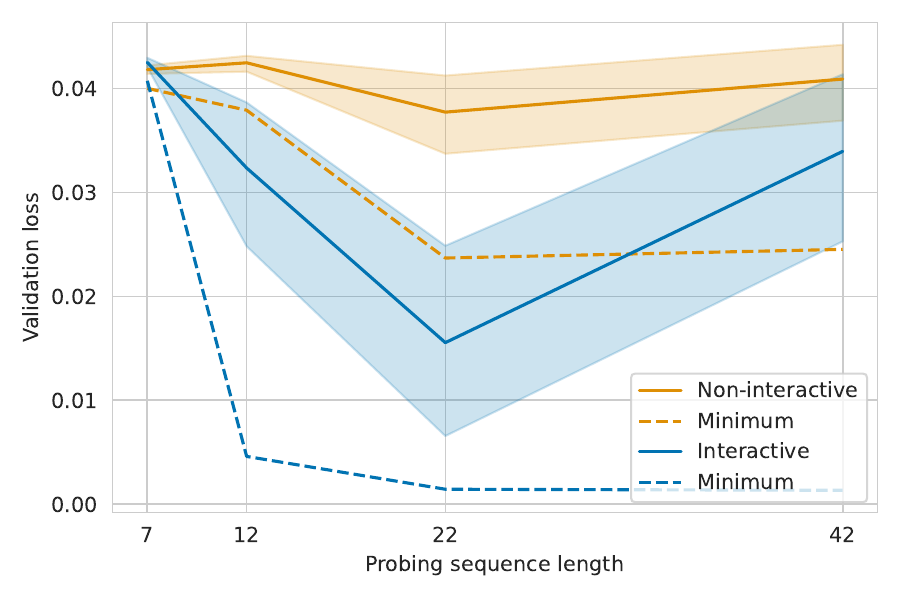}
        \caption{Formal Languages validation loss vs. length of probing sequences}
        \label{fig:bach_loss_vs_probing_seq_len}
    \end{minipage}
\end{figure*}


\begin{figure}[h]
\centering
\includegraphics[width=0.7\linewidth]{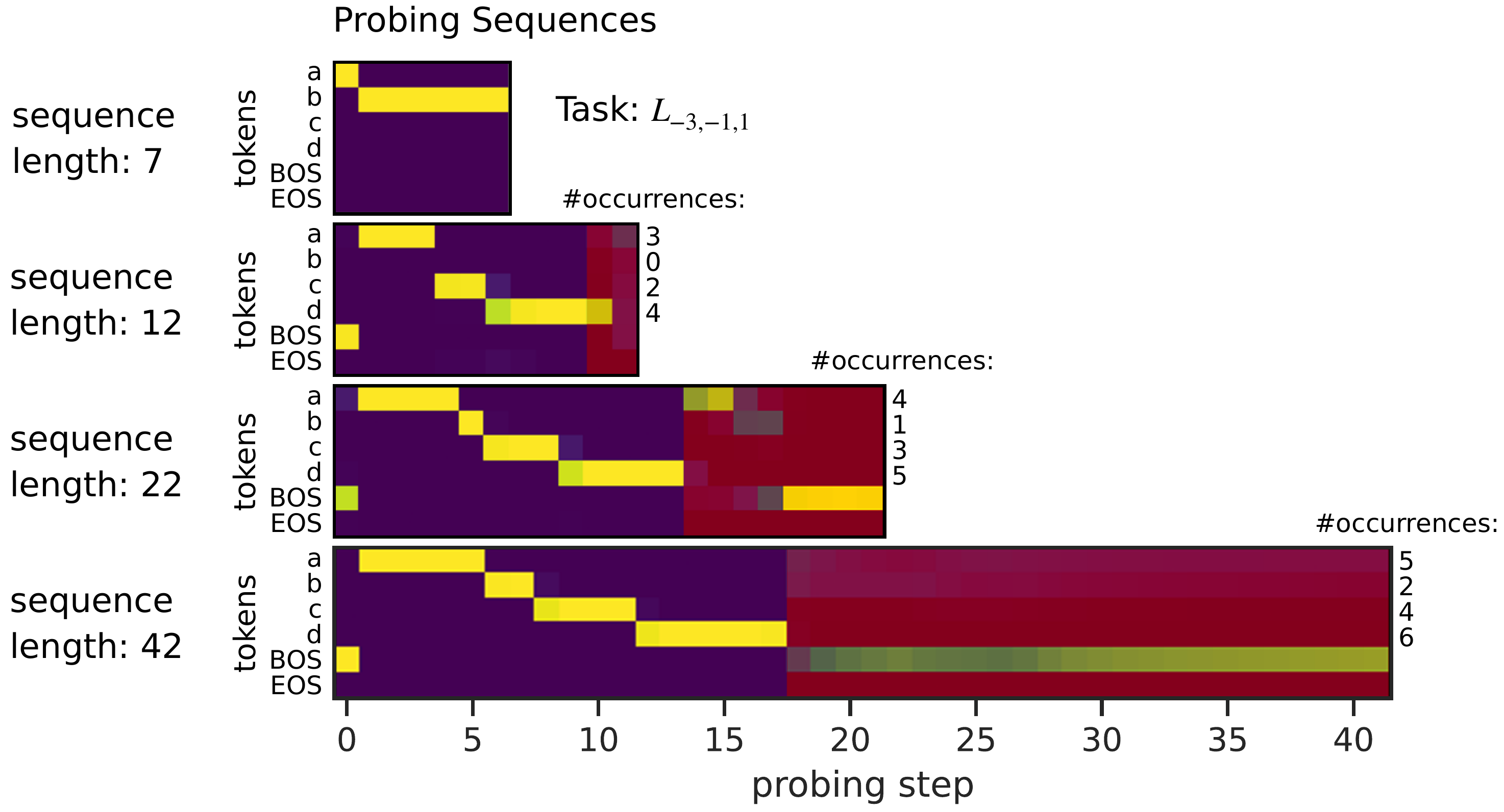}
\caption{Learned interactive probing sequences for a formal language}
\label{fig:bach_probing_sequences}
\end{figure}

\subsection{Hidden Neuron Permutation Invariance}
\label{app:permutation_invariance}

\begin{figure*}[h]
    \centering
    \begin{minipage}[t]{0.48\linewidth}
        \centering
        \includegraphics[width=1\linewidth]{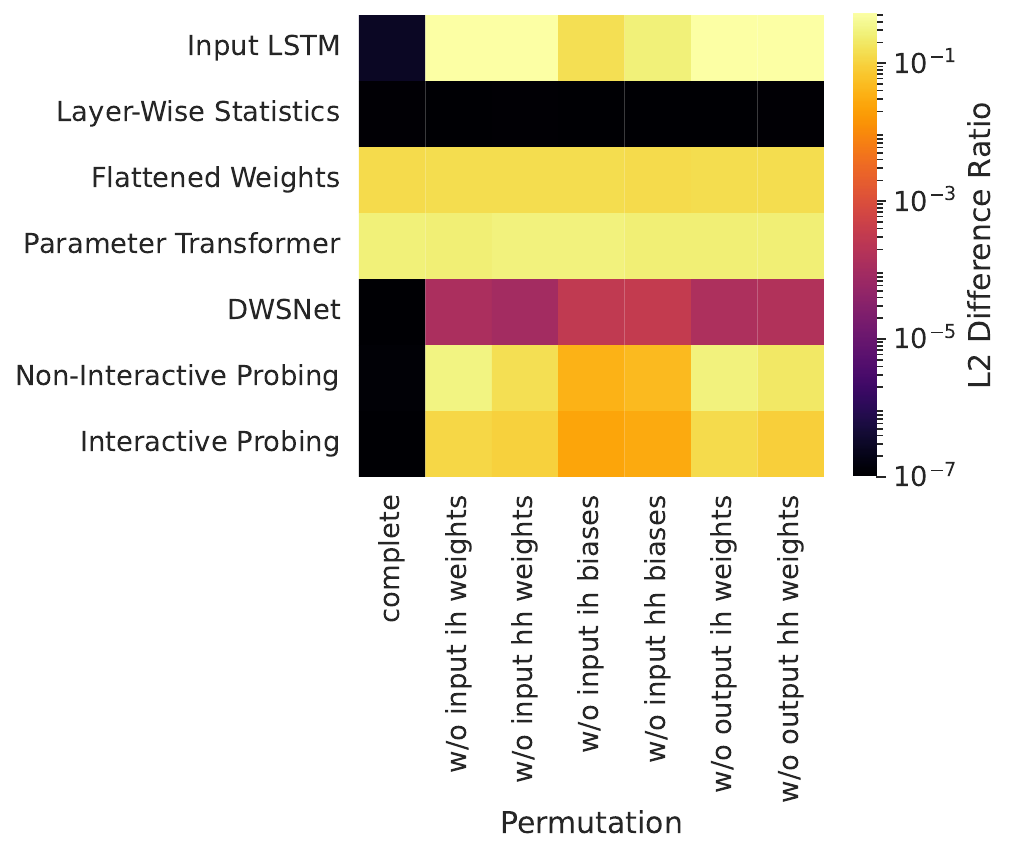}
        \caption{Various encoder architectures, with respect to either complete or incomplete permutation of hidden neurons, are analyzed. A black cell indicates that the output remains unchanged, whereas a bright cell signifies that the output changes to a similar extent as it would have with completely different weights. These results pertain to untrained encoders.}
        \label{fig:invariances_untrained}
    \end{minipage}\hfill
    \begin{minipage}[t]{0.48\linewidth}
        \centering
        \includegraphics[width=1\linewidth]{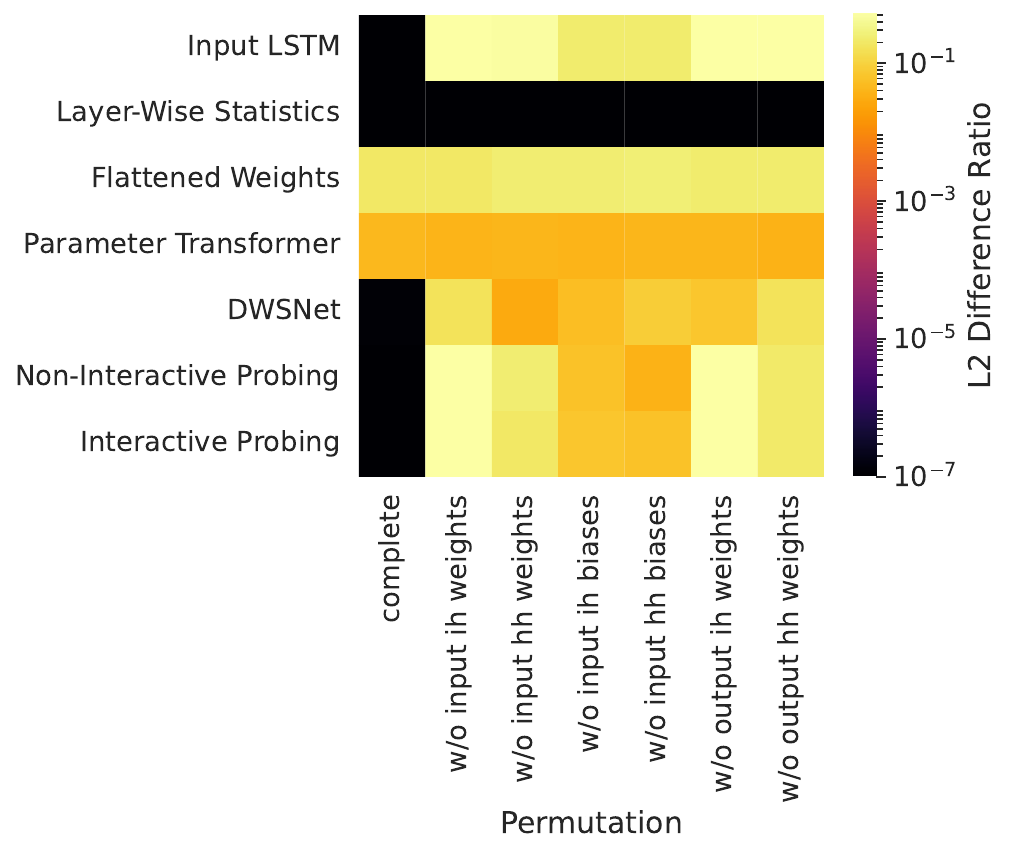}
        \caption{Invariance properties of the input LSTM $f_\theta$ and of fully trained encoders. Although the general invariance properties of the architectures remain unchanged, we observe that DWSNet becomes more sensitive to incorrect or incomplete permutations.}
        \label{fig:invariances_pretrained}
    \end{minipage}
\end{figure*}

Figures~\ref{fig:invariances_untrained} and \ref{fig:invariances_pretrained} illustrate the invariance of different encoder architectures towards correct (first column) and incorrect (other columns) hidden neuron permutations. Ideally, invariance should be present only in the first column. However, the Layerwise Statistics encoder shows invariance to both correct and incorrect permutations. In contrast, the Flattened Weights and Parameter Transformer encoders exhibit no invariance, neither before (Figure~\ref{fig:invariances_untrained}) nor after training (Figure~\ref{fig:invariances_pretrained}). 
During training, DWSNet becomes less invariant to wrong and incomplete permutations.
The methodology of the plot is explained below.

An LSTM consists of four gates: input gate, output gate, forget gate, and cell gate. These gates process the input data (from the previous layer) and the incoming hidden state. For a correct hidden neuron permutation in an LSTM, which preserves accuracy, all of the following elements must undergo the same permutation:
\begin{itemize}
    \item the rows of the four input-to-hidden weight matrices
    \item the rows of the four hidden-to-hidden weight matrices
    \item the four input-to-hidden bias vectors 
    \item the four hidden-to-hidden bias vectors
    \item the columns of the four input-to-hidden weight matrices of the next layer 
    \item the columns of the four hidden-to-hidden weight matrices
\end{itemize}

In Figure~\ref{fig:invariances_pretrained}, the top row labeled `Input LSTM' shows $\frac{||f_{\hat{\theta}}(x) - f_{\theta}(x)||_2}{|| f_\psi(x) - f_{\theta}(x) ||_2}$, where $\theta \in \Theta$ represents the original weights, $\tilde{\theta}$ is the (correct or incorrect, depending on the column in the figure) permutation of $\theta$ and $\psi \in \Theta$ is an entirely different set of weights.
The other rows show $\frac{||E(\hat{\theta}) - E(\theta)||_2}{|| E(\psi) - E(\theta)||_2}$ for different encoders $E$.
A bright cell indicates that the permutation significantly changes the result. A black cell means that the result for $\hat{\theta}$ is unchanged compared to the original result for $\theta$.
The encoders used in  Figure~\ref{fig:invariances_pretrained} have been trained on the Formal Languages dataset, which also provides $\theta$ and $\psi$.
Layerwise statistics features are invariant to all neuron permutations. The Flattened Weights Encoder and the Parameter Transformer show no signs of permutation invariance, even after training. DWSNet is, due to its construction, invariant only to complete permutations. During training it gets more sensitive to wrong and incomplete permutation. The probing encoders naturally inherit the invariance properties directly from the RNN $f$.

\subsection{Downstream Performance}

Table~\ref{tab:correlations} shows the Pearson correlation coefficients between pre-training validation loss, as defined in Equation~\ref{eq:loss}, and the downstream prediction losses for different properties, across all encoder architectures and random seeds. This analysis offers insight into the alignment between the pre-training objective and downstream applications. There is a high correlation with formal language task prediction and with Sequential MNIST performance prediction as well as in-distribution task prediction. A lower correlation does not necessarily imply that pre-training is less effective; it may simply indicate that moderate emulation performance is sufficient for good downstream prediction results. Figures~\ref{fig:downstream_bar_chart_bach} and \ref{fig:downstream_bar_chart_mnist} provide the complete results of downstream predictions for Formal Languages and Sequential MNIST data.

\begin{table}[h]
\caption{Correlation of pre-training validation loss with downstream prediction performance}    
\label{tab:correlations}
    \begin{center}
    \begin{small}
    \begin{tabular}{lcc}
        
        \toprule
        \textbf{Downstream Prediction} & \textbf{Validation} & \textbf{OOD Test} \\
        \midrule
        
        \textbf{Formal Languages} \\
        Performance & 0.509 & 0.505 \\ 
        Task & 0.985 & 0.987 \\ 
        \textbf{Sequential MNIST} \\
        Performance & 0.883 & 0.872 \\ 
        Task & 0.864 & 0.521 \\ 
        Generalization Gap & 0.776 & 0.642 \\ 
        Training Step & 0.762 & 0.486 \\
    \bottomrule
    \end{tabular}
    \end{small}
    \end{center}
\end{table}

\begin{figure}[h]
\centering
\includegraphics[width=0.8\linewidth]{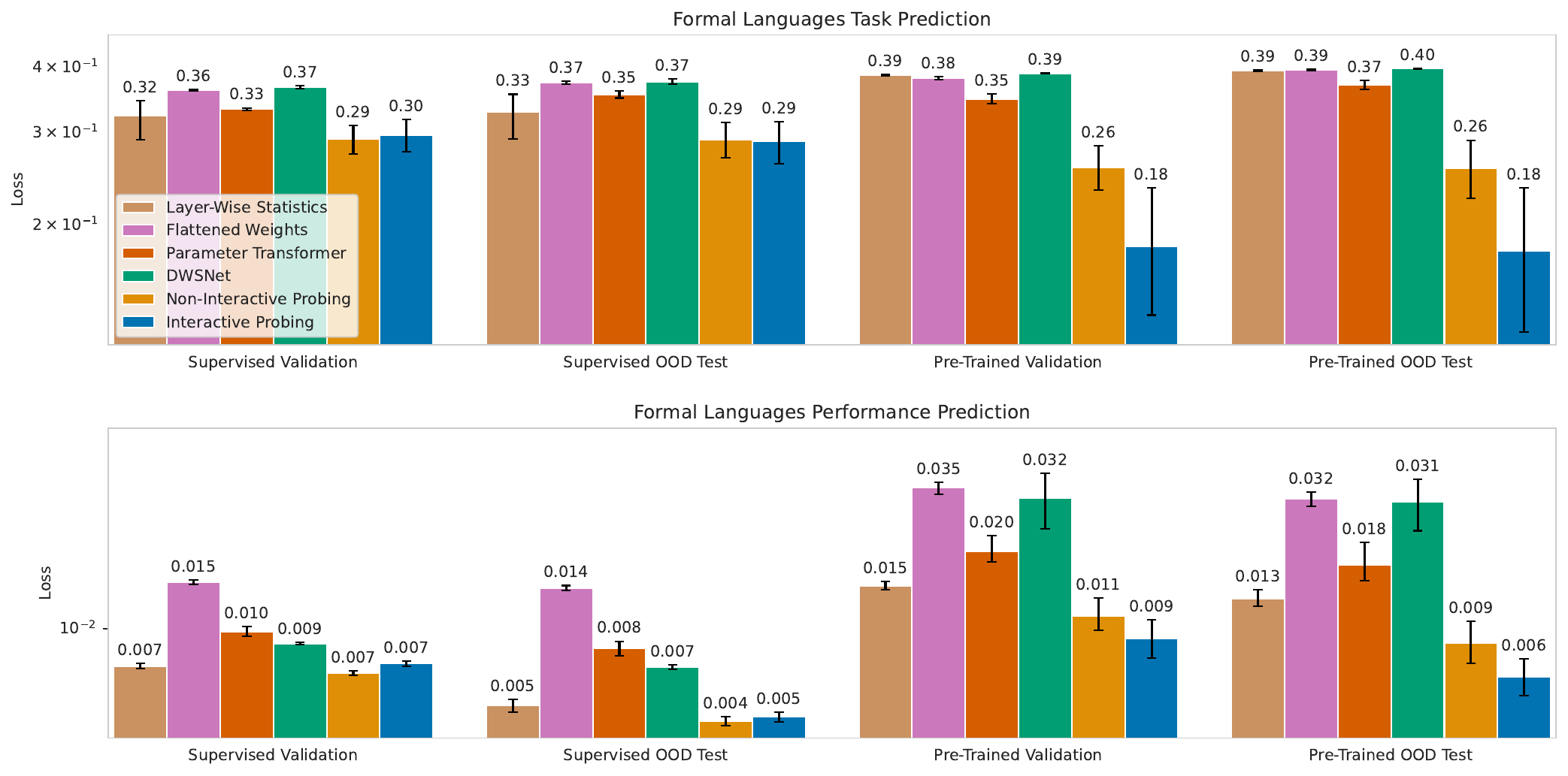}
\caption{Formal Languages downstream performance on task and performance prediction.}
\label{fig:downstream_bar_chart_bach}
\end{figure}

\begin{figure}[h]
\centering
\includegraphics[width=0.8\linewidth]{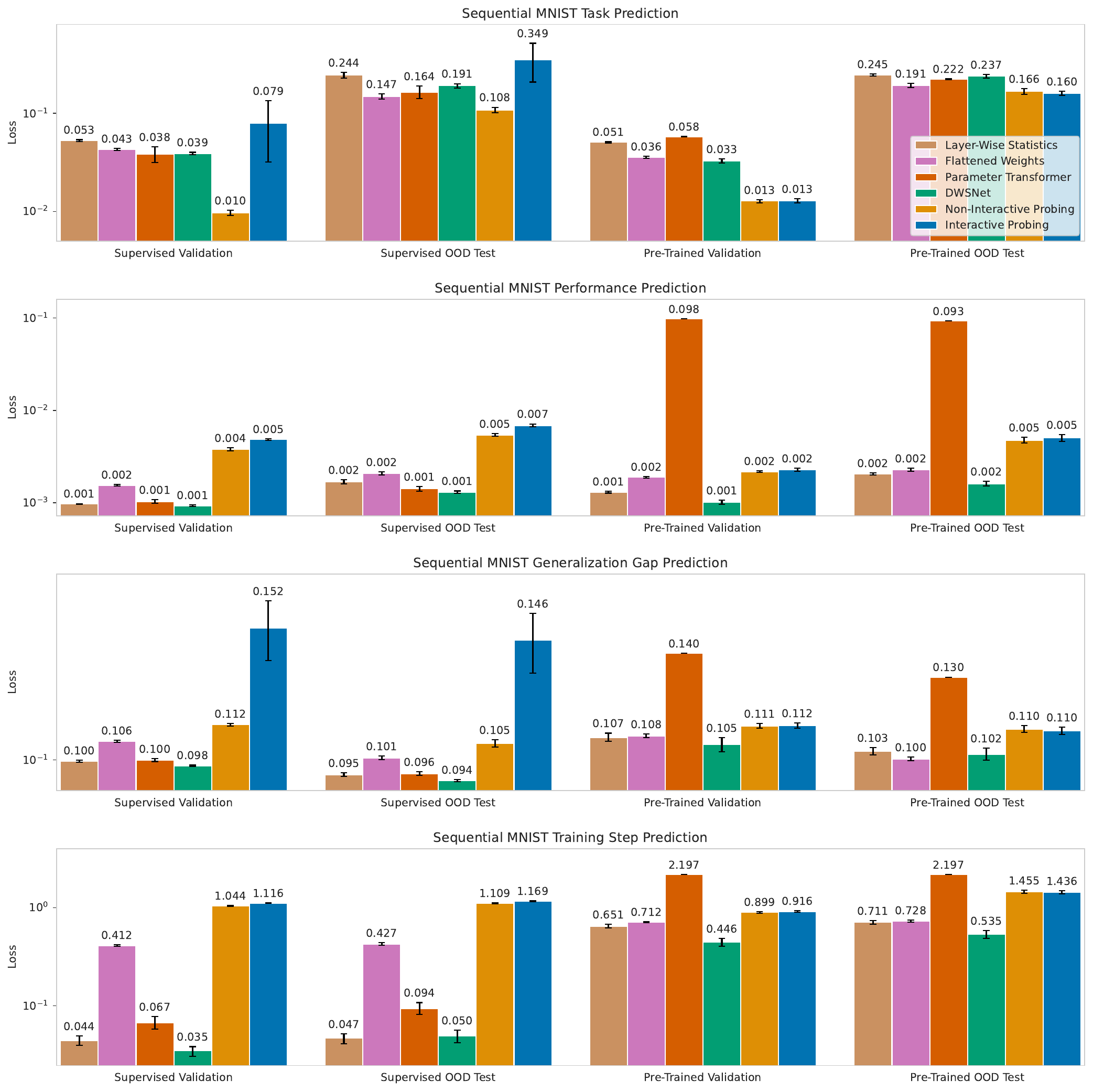}
\caption{Sequential MNIST downstream performance on task prediction, performance prediction, generalization gap, and training step prediction.}
\label{fig:downstream_bar_chart_mnist}
\end{figure}

\FloatBarrier
\clearpage

\subsection{Learned Embedding Spaces}
\label{app:embedding_viz}

\begin{figure}[!ht]
\centering
\includegraphics[width=0.95\linewidth]{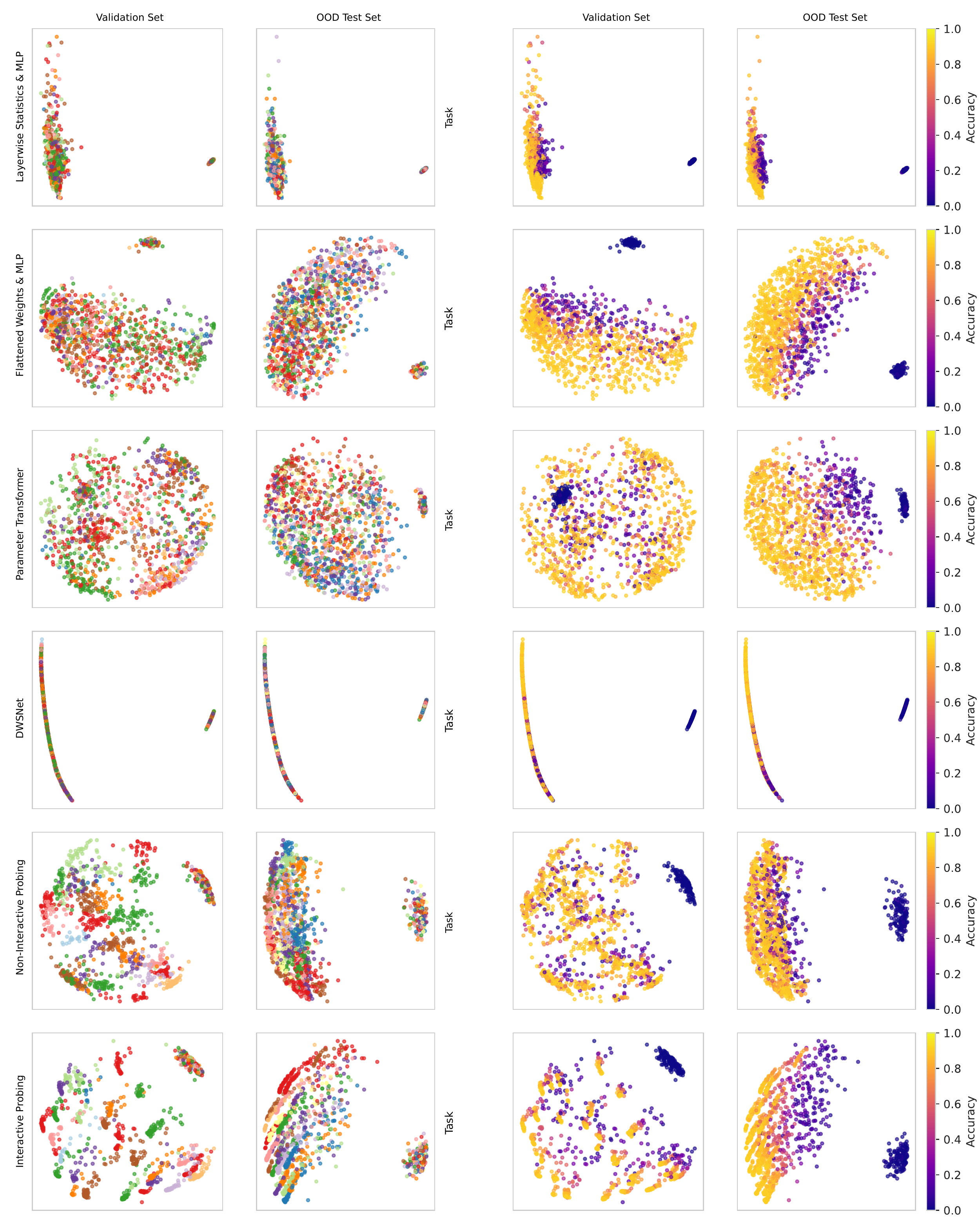}
\caption{Visualization of the embedding spaces for Formal Languages.}
\label{fig:embedding_space_viz_bach}
\end{figure}

\begin{figure}[h]
\centering
\includegraphics[width=0.95\linewidth]{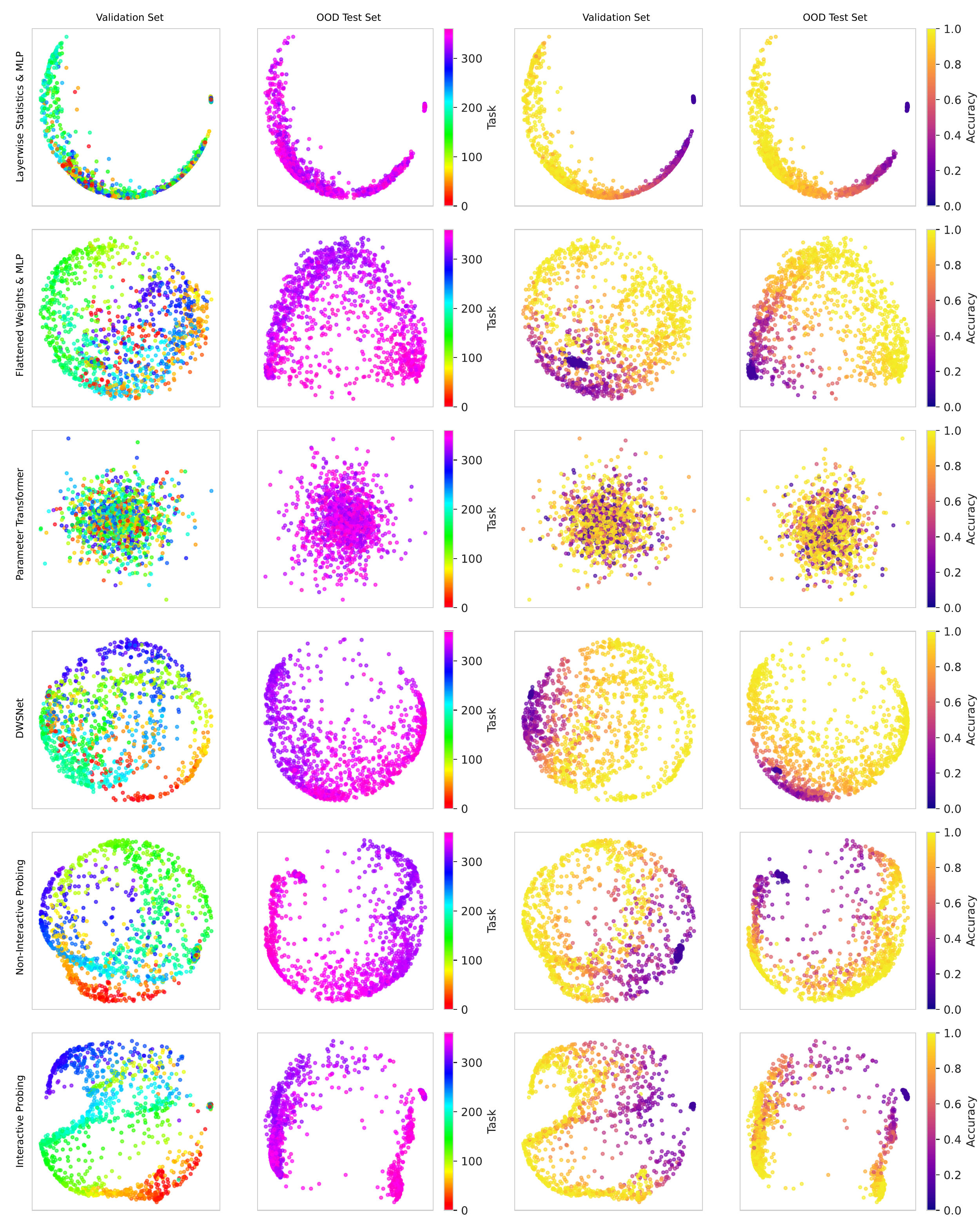}
\caption{Visualization of the embedding spaces for Sequential MNIST (colored by task and return).}
\label{fig:embedding_space_viz_mnist_task_return}
\end{figure}

\begin{figure}[h]
\centering
\includegraphics[width=0.95\linewidth]{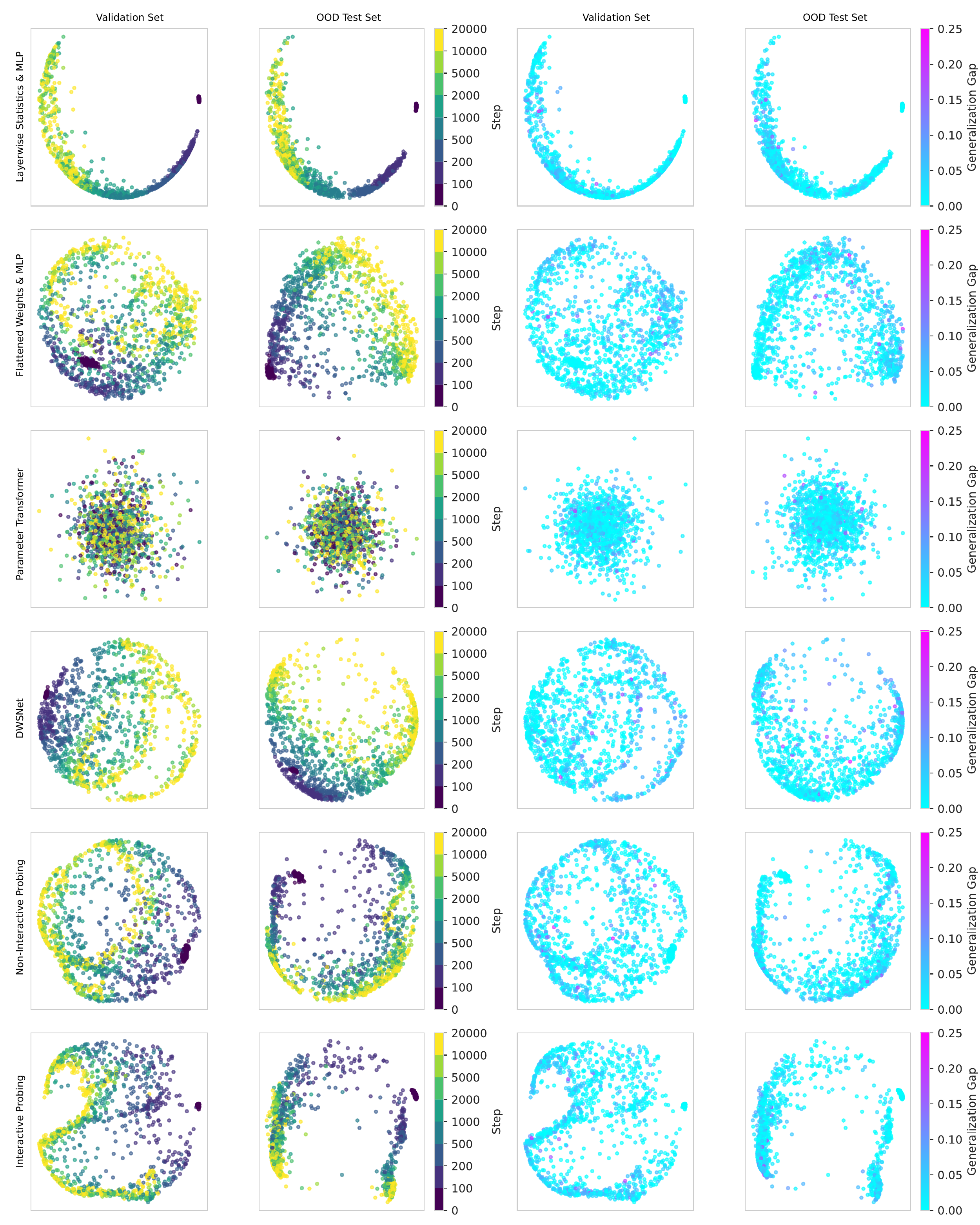}
\caption{Visualization of the embedding spaces for Sequential MNIST (colored by training step and generalization gap).}
\label{fig:embedding_space_viz_mnist_step_generalization}
\end{figure}



\end{document}